\documentclass[ejsv2, noshowframe, nojournaltitleforarxiv]{imsart}

\RequirePackage[numbers]{natbib}
\RequirePackage[colorlinks,citecolor=blue,urlcolor=blue]{hyperref}
\RequirePackage{graphicx}

\usepackage{amsfonts}
\usepackage[utf8]{inputenc}
\usepackage{t1enc}
\usepackage{nicefrac}
\usepackage{tocloft}
\usepackage{etoolbox}
\usepackage{enumitem}
\usepackage{indentfirst}
\usepackage{bbm}
\usepackage{caption}
\usepackage{lmodern}
\usepackage{nicefrac}
\usepackage{fancyhdr}
\usepackage[usenames,dvipsnames]{xcolor}
\usepackage{subcaption}
\usepackage{algorithm, algorithmic}
\usepackage{booktabs}
\usepackage{hyphenat}

\allowdisplaybreaks

\def\argmax{\operatornamewithlimits{arg\,max}}
\def\argmin{\operatornamewithlimits{arg\,min}}

\newcommand \dd[1]  { \,\textrm d{#1} }

\newcommand{\QX}{Q_{\scriptscriptstyle X}}

\newcommand{\BX}{\mathbb{X}}

\newcommand{\BY}{\mathbb{Y}}

\newcommand{\BA}{\mathbb{A}}

\newcommand{\BI}{\mathbb{I}}
\newcommand{\CX}{\mathcal{X}}

\newcommand{\CF}{\mathcal{F}}
\newcommand{\CA}{\mathcal{A}}

\newcommand{\CH}{\mathcal{H}}
\newcommand{\CN}{\mathcal{N}}

\def \CG{\mathcal{G}}
\def \CF{\mathcal{F}}
\def \CX{\mathcal{X}}

\def \CD{\mathcal{D}}

\usepackage{mathtools}

\DeclarePairedDelimiter\floor{\lfloor}{\rfloor}

\DeclareMathOperator {\sign}{sign}

\newcommand{\tr}{^\mathrm{T}}  
\newcommand{\RR}{\mathbb{R}}
\newcommand{\R}{\mathbb{R}}
\newcommand{\NN}{\mathbb{N}}

\newcommand{\EE}{\mathbb{E}}
\newcommand{\PP}{\mathbb{P}}

\newcommand{\norm}[1]{\left\lVert#1\right\rVert}

\newcommand\Tstrut{\rule{0pt}{2.6ex}}         
\newcommand\Bstrut{\rule[-0.9ex]{0pt}{0pt}}

\newtheorem{assumption}{A\hspace{-1.4mm}}
\newtheorem{assumptionp}{P\hspace{-0.5mm}}
\newtheorem{innercustomgeneric}{\customgenericname}
\providecommand{\customgenericname}{}
\newcommand{\newcustomtheorem}[2]{%
  \newenvironment{#1}[1]
  {%
   \renewcommand\customgenericname{#2}%
   \renewcommand\theinnercustomgeneric{##1}%
   \innercustomgeneric
  }
  {\endinnercustomgeneric}
}

\newcustomtheorem{customthm}{Theorem}
\newcustomtheorem{customlemma}{Lemma}

\allowdisplaybreaks

\startlocaldefs
\theoremstyle{plain}

\newtheorem{theorem}{Theorem}[section]
\newtheorem{lemma}[theorem]{Lemma}
\newtheorem{definition}[theorem]{Definition}

\theoremstyle{remark}

\endlocaldefs

\begin{document}
\begin{frontmatter}
\title{Resampled Confidence Regions\\
with Exponential Shrinkage for the\\ 
Regression Function of Binary Classification}
\runtitle{Confidence Regions for the Regression Function of Binary Classification}

\begin{aug}

\author[1, 2]{\fnms{Ambrus}~\snm{Tamás}\ead[label=e1]{tamasamb@sztaki.hu}}
\and
\author[1, 2]{\fnms{Balázs Csanád}~\snm{Csáji}\ead[label=e2]{csaji@sztaki.hu}}

\address[1]{Institute for Computer Science and Control (SZTAKI), Hungarian Research Network (HUN-REN)\printead[presep={,\ }]{e1,e2}}

\address[2]{Dept.\ of Probability Theory and Statistics, Institute of Mathematics, E{\"o}tv{\"o}s Lor{\'a}nd University (ELTE)}
\runauthor{A. Tamás and B.Cs., Csáji}
\end{aug}

\begin{abstract}
The regression function is one of the key objects of binary classification, since it not only determines a Bayes optimal classifier, hence, defines an optimal decision boundary, but also encodes the conditional distribution of the output given the input. In this paper we build distribution-free confidence regions for the regression function for any user-chosen confidence level and any finite sample size based on a resampling test. These regions are abstract, as the model class can be almost arbitrary, e.g., it does not have to be finitely parameterized. We prove the strong uniform consistency of a new empirical risk minimization based approach for model classes with finite pseudo-dimensions and inverse Lipschitz parameterizations. We provide exponential probably approximately correct bounds on the $L_2$ sizes of these regions, and demonstrate the ideas on specific models. Additionally, we also consider a k-nearest neighbors based method, for which we prove strong pointwise bounds on the probability of exclusion. Finally, the constructions are illustrated on a logistic model class and compared to the asymptotic ellipsoids of the maximum likelihood estimator.
\end{abstract}



\end{frontmatter}

\section{Introduction}

Classification or pattern recognition is one of the fundamental problems of statistical learning theory \citep{Vapnik1998}, and it is widely studied in several fields, such as machine learning, computer vision, system identification, signal processing and information theory. Assuming we want to minimize the probability of misclassification, an intrinsic goal is to estimate the regression function, because it not only determines a Bayes optimal classifier, but also encodes the conditional misclassification probability given the input.
Standard methods typically only provide point estimates for the regression function, which usually equal to the target model with probability zero in case of finite samples. Hence, it is crucial to quantify the uncertainty of an estimator, for example, by providing error bounds for a given point estimate. In this paper our goal is to construct these in the form of confidence regions.

There are several methods to construct region estimates, for example, confidence intervals and ellipsoids, but most of these approaches suffer from theoretical limitations. In the area of parametric statistics, under strong assumptions on the observations, the distribution of the estimate can be derived. For example, confidence intervals for the expected value of a normal distribution can be constructed based on the sample mean and variance.

In statistical learning the distribution of the sample is unknown, therefore, distribution-free methods are sought. 
The asymptotic theory of distribution-free regression function estimation is rich \citep{gyorfi2002distribution}. For example, k-nearest neighbors (kNN) methods and Nadaraya-Watson type kernel estimators are proved to be strongly universally $L_2$ consistent \citep{devroye1994strong, walk2002almost}, however, the convergence rates of these methods can be arbitrarily slow \citep{cover1968rates} and for nonparametric families of distributions the rates are usually way too conservative (for example, minimax) to be useful in practice. With an asymptotic approach, we can approximate the distribution of the scaled estimation error with its limiting distribution, which can often be characterized by using the central limit theorem. However, asymptotic methods can perform poorly for small samples; the provided guarantees are only approximate for finite datasets.

In finite sample settings, a standard approach is to rely on concentration inequalities, such as the results in \cite{hoeffding1994probability, bernstein1937modi, bennett1962probability}. These yield a variety of probably approximately correct (PAC) bounds, however, they are not fully data-driven, e.g., they still need information on moments or ranges. An adaptive martingale approach was developed in \citep{Waudby2023} to derive uniform concentration bounds for the mean of a random variable, under a boundedness assumption. Distribution-free binary classification was studied in \cite{barber2020distribution} in the framework of conformal inference \citep{vovk2005algorithmic}. Barber proves an explicit lower bound on the expected width of distribution-free confidence intervals for the conditional class probability in case of binary outputs and constructs a procedure that achieves this length approximately \cite{barber2020distribution}.

In this paper our goal is to build data-driven, nonparametric, non-asymptotically exact confidence regions for the regression function of binary classification and prove strong consistency results as well as exponential PAC bounds  under very mild statistical assumptions. The approach was motivated by the Sign-Perturbed Sums (SPS) method, which was introduced in system identification for linear regression \citep{csaji2015closed, csaji2014sign}. SPS constructs non-asymptotically exact confidence regions and it is proved to be strongly consistent \citep{weyer2017asymptotic}. Extensions of SPS for classification were suggested in \citep{csajitamas2019, tamas2021}, which we improve in several ways.

We present resampling methods that make inference based on alternatively (artificially) drawn samples. We argue that resampling methods can provide universal and adaptive stochastic guarantees for binary classification. Similarly to Monte Carlo tests, we generate samples from a test distribution and compare suitable statistics of the alternative datasets to the original sample \citep{zhu2006nonparametric}. We define the confidence sets via rank tests, thus, they are presented in an abstract form, i.e., for each candidate model a test decides whether it is included in the region for a given confidence level. We emphasize the confidence region perspective and show how to reduce the cost of the alternative data generation process to a constant and still be able to test any candidate. We introduce an empirical risk minimization (ERM) based approach and prove that it builds strongly uniformly consistent confidence regions. We prove an exponential PAC bound for the $L_2$ size of ERM-based confidence sets. We also present a new consistency theorem for kNN-based algorithms under milder assumptions than in \citep{csajitamas2019}. Additionally, a pointwise exponential PAC bound is proved for the kNN-based algorithm, as well. We prove the strong consistency of our confidence regions in an $L_2$ space, under no assumptions on the marginal distribution of the inputs. This seemingly contradicts the results of \cite{barber2020distribution}, however, our method does not construct confidence intervals with vanishing width for regression function values, but tests the functions themselves instead. Recall that $L_2$ convergence neither implies pointwise nor uniform convergence, in general. Finally, the methods are validated through numerical experiments on logistic models, and the constructed regions are compared to the asymptotic confidence ellipsoids around the maximum likelihood estimator (MLE). 

\section{Binary Classification}

Let $(\BX, \CX)$ be a measurable input space and $\BY=\{+1,-1\}$ be the binary output space. Let $\mathcal{D}_0= \{(X_i,Y_i)\}_{i=1}^n$ denote an i.i.d.\ sample from the unknown joint distribution, $Q_{\scriptscriptstyle{X,Y}}$ or simply $Q$, of the input-output pair $(X,Y) \in \BX \times \BY$, for $i\in [\hspace{0.3mm}n\hspace{0.3mm}]\doteq \{1,\dots, n\}$.

Measurable $\BX \to \BY$ type functions are called classifiers or decision rules.  Let us apply a measurable loss function $L:\BY \times \BY \to [\hspace{0.3mm}0,\infty)$ that penalizes label mismatch. We recall the zero-one loss, $L(y_1,y_2)\doteq \BI ( y_1 \neq y_2)$ for $y_1,y_2 \in \BY,$ where $\BI$ is an indicator function. The expected loss of  classifier $\phi$ is $R(\phi) \doteq \EE[\hspace{0.3mm} L(\phi(X), Y)\hspace{0.3mm}]$, which is also called the Bayes risk. Typically the goal of classification is to minimize this quantity. For the zero-one loss, the Bayes risk is the probability of misclassification. Since the joint distribution of $(X,Y)$ is unknown, minimizing the Bayes risk is challenging, and it is based on empirical estimates.

It is known that, assuming binary outputs, the joint distribution, $Q$,  is determined by the marginal distribution of the inputs, $\QX$, and the {\em regression function} $f^*(x) \doteq \EE[\hspace{0.3mm}Y|X=x\hspace{0.3mm}]$. Furthermore, for the zero-one loss, a Bayes (optimal) classifier takes the form of $\phi^*(x) \doteq \sign( f^*(x))$. The regression function not only determines an optimal classifier, but also encodes the misclassification probabilities for every input. That is why $f^*$ is a key object for binary classification. Henceforth, in this paper we aim at building confidence regions for $f^*$.

\section{Resampling for Classification}

In this section we present a resampling framework to build distribution\hyp free, non\hyp asymptotic confidence regions, which contain the regression function of binary classification with a prescribed (rational) confidence level. Later, we consider two specific resampling algorithms to build exact confidence sets for the regression function under mild statistical assumptions.

The core assumptions of the framework are as follows, cf. \cite{csajitamas2019}:
\begin{assumption}\label{ass:b1}
    The given sample, $\CD_0=\{(X_i,Y_i)\}_{i=1}^n$, is i.i.d. from distribution $Q\vspace{-1mm}$.
\end{assumption}
\begin{assumption}\label{ass:b2}
    A class of candidate regression functions is available, indexed by an arbitrary set of ``parameters'', which contains $f^*$, that is,
	$f^* \in \CF_0 \, \doteq \, \big\{\, f_{\theta}: \BX \to [\,-1,+1\,]\, \mid\, \theta \in \Theta\, \big\}.\vspace{-1mm}$
\end{assumption}
\begin{assumption}\label{ass:b3}
    The parameterization is injective in the $\mathcal{L}_2(\QX)$ sense, i.e., for all $\theta_1 \neq \theta_2 \in \Theta:$
    \vspace{-1mm}
		\begin{equation}
			\|\, f_{\theta_1} - f_{\theta_2} \|^2_{\scriptscriptstyle Q} \, \doteq \int_{\BX} (f_{\theta_1}(x)-f_{\theta_2}(x))^2 \dd\, \QX(x) \, \neq \, 0.
			\label{eq:L2-distance-PX}
	\end{equation}
\end{assumption}

Let $\theta^*$ denote the parameter which corresponds to the true regression function, that is $f_{\theta^*} = f^*$. The true parameter is well-defined, because of A\ref{ass:b3}.
Although we refer to $\Theta$ as ``parameter space'', it can be almost any abstract set (a metric structure will be assumed for the uniform consistency results), e.g., $\Theta$ can be an infinite dimensional vector space, moreover, the functions themselves can be the ``parameters''.
A confidence region is a random subset of $\Theta$ which covers $\theta^*$ with a user-chosen probability. In the section that follows abstract region estimators are defined as the accepted parameters of a rank test.

One of the main observations needed for our approach is that a regresssion function candidate determines the conditional distribution of the outputs given the inputs, that is
\begin{align}
	\label{eq:regression-function}
	\PP(\,Y(\theta) = \pm1\; |\; X = x\,)\, =\, \big(1 \pm f_\theta(x)\big)/2,
\end{align}
for all $\theta \in \Theta$. Notice that $Y$ has the same distribution as $Y(\theta^*)$.
Our idea is to generate alternative outputs for a given candidate model, then compare the alternative datasets to the original sample with a similarity measure to test the candidate's suitability. The comparison is performed via rank statistics, cf. Definition \ref{def:ranking-function-1}.

\begin{definition}[ranking function]\label{def:ranking-function-1}
	Let $\BA$ be a measurable space. A (measurable) function $\psi : \BA^m \to [\,m\,]$ is called a	{ranking function} if $\forall(a_1, \dots, a_m) \in \BA^m$ it satisfies P\ref{ass:p1} and P\ref{ass:p2}$\,:$
        \begin{assumptionp}\label{ass:p1}
            For all permutation $\tau$ on set $\{2,\dots, m\}$, 
			$\psi\big(\,a_1, a_{2}, \dots, a_{m}\,\big)\; = \;
			\psi\big(\,a_1, a_{\tau(2)}, \dots, a_{\tau(m)}\,\big),$
		that is, function $\psi$ is invariant w.r.t.\ reordering the last $m-1$ terms of its arguments.\vspace*{-1mm}
        \end{assumptionp}
        \begin{assumptionp}\label{ass:p2}
        For all indices $i,j \in  [\hspace{0.3mm}m\hspace{0.3mm}]$,
		if $a_i \neq a_j$, then we have
	    $\psi\big(\,a_i, \{a_{k}\}_{k\neq i}\,\big)\, \neq \;\psi\big(\,a_j, \{a_{k}\}_{k\neq j}\,\big),$
		where the simplified notation is justified by property P\ref{ass:p1}.
        \end{assumptionp}
\end{definition}

The (random) value of a ranking function is called the {\em rank}. The main observation about the rank is given by the lemma that follows \cite[Lemma 1]{csajitamas2019}:
\begin{lemma}\label{lemma:uniform-rank-distribution}
	{Let $\xi_1, \xi_2, \dots, \xi_m$ be 
    almost surely pairwise different exchangeable random elements taking values in an arbitrary measurable space, and 
    let  $\psi$ be a ranking function. Then, the rank, $\psi(\xi_1, \dots, \xi_m)$, is distributed uniformly on $[\hspace{0.3mm}m\hspace{0.3mm}]$.}
\end{lemma}

Lemma \ref{lemma:uniform-rank-distribution}, which is an important observation about the rank of exchangeable random elements, will be our main tool to prove that the (user-chosen) confidence level of the constructed confidence region is exact for every finite sample size.
We emphasize that Lemma \ref{lemma:uniform-rank-distribution} is distribution-free, i.e., the marginal distribution of the given random elements, $\{\xi_j\}_{j=1}^m$, can be arbitrary. In addition, the elements that we compare do not need to be fully independent, only their exchangeability is required, which is more general than the i.i.d.\ assumption.

We use a ranking function $\psi$ on the resampled datasets. The given sample, $\CD_0$, is included in $(\BX \times \BY)^n$. Our framework generates $m-1$ alternative random elements from this space for any candidate model, i.e., $\CD_j(\theta) \doteq \{(X_i,Y_{i,j}(\theta))\}_{i=1}^n$, where $Y_{i,j}(\theta) =\sign(f_\theta(X_i) + U_{i,j})$ with i.i.d.\ uniform variables $\{U_{i,j}\}$ on interval $[-1,1]$ for $i\in[n]$ and $j\in[m-1]$. That is, the conditional distribution of $Y_{i,j}(\theta)$ w.r.t.\ $X_i$ is given by $f_\theta$. Hence, we consider $m$ elements in $(\BX \times \BY)^n$. To ensure pairwise difference, which is a technical assumption in Lemma \ref{lemma:uniform-rank-distribution}, we extend the datasets with the elements of a uniformly sampled random permutation $\pi: [m] \to [m]$, i.e., we let $\CD_0^\pi(\theta)\doteq \big( \CD_0, \pi(m)\big)$ and $\CD_j^\pi(\theta)\doteq \big( \CD_j(\theta), \pi(j)\big)$ for $j \in [m-1]$. Henceforth, we construct $\psi: \big((\BX \times \BY)^n \times [m]\big)^m \to [m]$ type ranking functions.

\begin{algorithm}[t]
    \caption{Initialization}\label{alg:initialization}
    \hspace*{-8.87cm}\textbf{Input:} rational confidence level $\gamma \in (0,1)$\\[1mm]
    \begin{algorithmic}[1]\label{hypothesis_algo2}
         \hrule
         \vspace{1mm}
         \STATE Select integers $1 \leq q_1 < q_2 \leq m$ such that 
         $\gamma = (q_2-q_1 +1)/m.$
         \STATE Generate i.i.d.\ variables, $\mathcal{U}\doteq \{U_{i,j}\}\text{\;for\;}i \in [n], j \in [m-1]\text{\;with\;}$
              $U_{i,j} \sim \text{Uniform}(-1,1).\hspace*{-1mm}$
         \STATE Generate a random permutation $\pi$ uniformly from the symmetric group over $[m]$.
    \end{algorithmic}
\end{algorithm}
The confidence region construction for an arbitrary ranking function $\psi$ consists of two parts. First, in the initialization phase, in Algorithm \ref{alg:initialization}, the integer hyperparameters ($q_1$, $q_2$ and $m$) are selected w.r.t.\ the prescribed (rational) confidence level, $\gamma$, and a stem sample is generated. Then, in Algorithm \ref{alg:rank-statistic}, rank statistics are calculated on the parameter space with the help of $\psi$. Finally, we construct the abstract confidence region based on these ranks by
\begin{equation}
	\Theta_n^{\psi}\, \doteq\, \big\{ \, \theta \in \Theta \;|\;  q_1 \leq \psi\big(\,\CD^{\pi}_0(\theta), \{ \CD^{\pi}_j(\theta) \}_{j \neq 0}\,\big) \leq q_2\, \big\}.
\end{equation}

\begin{algorithm}[t]
    \caption{Rank Statistic}\label{alg:rank-statistic}
    \hspace*{-3.2cm}\textbf{Inputs:} 
    parameter $\theta$, samples $\CD_0$ and $\mathcal{U}$,
    ranking function $\psi$, integer $m$, permutation $\pi$\\[1mm]
    \begin{algorithmic}[1]\label{hypothesis_algo2}
         \hrule
         \vspace*{1mm}
         \STATE Let $Y_{i,j}(\theta) \doteq \sign(f_\theta(X_i) + U_{i,j})$ for $i \in [n]$ and $j \in [m-1]$ be the alternative outputs for model $f_\theta$. For simplicity let $Y_{i,0}(\theta)= Y_i$ for $i \in [n]$.
         \STATE Let $\CD_j(\theta) \doteq \{(X_i,Y_{i,j}(\theta))\}_{i=1}^n$ be an alternative sample for $j \in [m-1]$.
         \STATE Le $\CD_j^\pi(\theta)\doteq \big( \CD_j(\theta), \pi(j)\big)$ be the extended alternative sample for $j \in [m-1]$ and for notational simplicity let $\CD_0^\pi(\theta) \doteq ( \CD_0, \pi(m)),$ 
         for all $\theta\in \Theta$.
         \STATE Return the rank statistic of parameter $\theta$, that is return
         $\psi(\,\CD_0^\pi(\theta), \{\CD_j^\pi(\theta)\}_{j=1}^{m-1}\,).$\vspace{-1mm}
    \end{algorithmic}
\end{algorithm}

 We show that $\Theta_n^\psi$ is an {\em exact} confidence region for $\theta^*$ with confidence level $\gamma$.
    Let us fix an index $i \in [n]$ and take the observations that follows. First, we note that
	\begin{align}
		\PP(\,Y_{i,j}(\theta)=\pm1\,|\,X_i\,)\, = \,
        (1 \pm f_\theta(X_i))/2
	\end{align}
	and consequently $Y_{i,j}(\theta)$ is generated from the conditional distribution with respect to $X_i$ determined by candidate regression function $f_\theta$ for all $j \in [m-1]$. Second, clearly $Y_{i,0}$ has the same conditional distribution with respect to $X_i$ as $Y_{i,j}(\theta^*)$. In addition $\{ \,Y_{i,j}(\theta^*)\,\}_{j=0}^{m-1}$ are all conditionally i.i.d.\ with respect to $X_i$, because $\{U_{i,j}\}_{j=1}^{m-1}$ are independent of each other and $Y_{i,0}$. We conclude that $\CD_0,\CD_1(\theta^*), \dots, \CD_{m-1}(\theta^*)$ are conditionally i.i.d. with respect to $\{X_i\}_{i=1}^n$, thus they are also exchangeable. 
	Hence, the application of Lemma \ref{lemma:uniform-rank-distribution} yields the general theorem that follows, which is one of        the main building blocks of the framework:
	\begin{theorem}
		\label{theorem:exact-confidence}
		Assume that A\ref{ass:b1}, A\ref{ass:b2} and A\ref{ass:b3} hold. Then, for any ranking function $\psi$, integers $1\, \leq\, q_1\, \leq\, q_2\, \leq\, m$ and sample size $n \in \mathbb{N}$, we have\hspace{0.3mm}
		$\PP\big(\, \theta^* \in \Theta_{n}^{\psi}  \, \big)\, = \, (q_2-q_1+1)/m.$
	\end{theorem}

    Note that Theorem \ref{theorem:exact-confidence} holds for any finite dataset. If interpreted as a hypothesis test, Theorem \ref{theorem:exact-confidence} quantifies exactly the probability of type I error.
    Any rational coverage level can be achieved accurately by using appropriate hyperparameters, which are under our control. 
    If needed, irrational coverage levels could also be achieved by using additional randomization.
    We allow parameterizing the regression function, while the distribution of the inputs can be arbitrary, hence our approach can be applied in a semi-parametric setting, too.
	On the other hand, our method provides valid confidence sets on a completely distribution-free fashion, e.g., if we consider the model class of all possible regression functions.

 The application of the proposed framework requires an appropriate ranking function. We develop rankings that produce consistent confidence regions, which eventually exclude every ``false'' parameter. We aim at ordering datasets $\{ \CD_j^{\pi}(\theta)\}_{j=0}^{m-1}$ with the help of $\psi$. The extended datasets are from a high-dimensional space, where a total order is not given in general. For this purpose we introduce real-valued reference variables of the form
	$Z_n^{(j)}(\theta)\, \doteq\, T(\CD_j(\theta),\theta),$
	for $j=0,\dots,m-1$, where $T: (\BX \times \BY)^n \times \Theta \to \RR$ (measurable for all $\theta \in \Theta)$. To sort $\{Z_n^{(j)}(\theta)\}_{j=0}^{m-1}$ we use the total order $\prec_\pi$ defined by
	$Z_n^{(j)}(\theta)\, \prec_\pi\, Z_n^{(k)}(\theta)$ if and only if $Z_n^{(j)}(\theta) < Z_n^{(k)}(\theta)$ or $\big(Z_n^{(j)}(\theta)\,=\,Z_n^{(k)}(\theta) \text{ \,and\, } \pi(j) \,<\, \pi(k)\big)$.
	With these notations let
	\begin{equation}\label{eq:ranking-function}
		\begin{aligned}
			&\psi\,\big(\,\CD_0^\pi(\theta), \{\CD_{j}^\pi(\theta)\}_{j=1}^{m-1}\,\big)\,\doteq\, 1 + \sum_{j=1}^{m-1} \BI(Z_n^{(0)}(\theta) \succ_\pi Z_n^{(j)}(\theta)).
		\end{aligned}
	\end{equation}
	Clearly, $\psi$ satisfies P\ref{ass:p1} and P\ref{ass:p2}. In the following sections we design real-valued statistics and reference variables that guarantee consistency for the constructed confidence regions beside the exact coverage probability. Note that $T$ can depend on the model class as well, however, in this paper $\CF_0$ is fixed, therefore we do not indicate this dependence in the notation.

    \section{Point Estimator Based Constructions}
	In this section we present a general approach to define reference variables. Our idea is to estimate the regression function from each available dataset, from the original one as well as the resampled ones, and compare their empirical performance w.r.t.\ the candidate function. 
	Any estimator of the regression function can be used, but in this paper we focus on two particular approaches:
    an ERM-based and a kNN-based technique. Both approaches construct exact confidence regions for any user-chosen confidence level and we prove that they also have strong asymptotic guarantees. Sufficient conditions are presented for strong uniform consistency for the ERM-based construction and strong pointwise consistency is proved for the kNN-based scheme. 
	We make the additional assumption that follows:
        \begin{assumption}\label{ass:b4}
            (Euclidean input space condition) $\BX \subseteq \RR^d$
        \end{assumption}
	
	\subsection{Empirical Risk Minimization}\label{sec:ERM-based}
	
	It is well-known that the regression function is a risk minimizer; it minimizes the true expected squared loss among all measurable functions, 
	$f^* \,=\, \argmin\nolimits_f\, \EE \big[(f(X)-Y)^2\big].$
	It motivates the application of the ERM principle, which given a model class $\CF$ and a sample $\{(X_i,Y_i)\}_{i=1}^n$ estimates the regression function by choosing
	\begin{equation}
		{f}_{n}\, \doteq\, \argmin_{f \in \CF} \frac{1}{n}\sum_{i=1}^n (f(X_i)-Y_i)^2,\vspace{1mm}
	\end{equation}
	where the quantity in the right hand side is called the empirical risk. Note that we assume the existence and the measurability of ${f}_{n}$, however, we do not require its uniqueness.
	
	When $f^*$ is included in a model class $\CF$, the consistency of ERM estimates relies on the uniform convergence of the empirical risk to the true risk, i.e., whether we have, as $n\to \infty$,
    \begin{equation}\label{eq:ulln}
		\sup_{f \in \CF}\hspace{0.3mm} \Bigg| \frac{1}{n} \sum_{i=1}^n (f(X_i)-Y_i)^2 -\EE \big[ (f(X) -Y)^2\big]\Bigg| \xrightarrow{\;a.s.\;}0.
	\end{equation}
	These type of convergences are called the (strong) uniform laws of  large numbers (ULLN) and are in the core of statistical learning theory \citep{Vapnik1998, gyorfi2002distribution}.
	The measurability of the supremum in \eqref{eq:ulln} needs to be verified for the model class as it may contain uncountably many elements. For simplicity and to avoid digressions, we always assume the measurability of the arising supremums without further notice.

    Classical approaches define a complexity measure to quantify the capacity of the model class and prove ULLN by posing bounds on the complexity.
        We use the concept of pseudo-dimension to restrict the expressivity of $\CF_0$. This notion is based on the celebrated Vapnik-Chervonenkis dimension (VC dimension), whose definition is recalled in Appendix \ref{app:theorems}.    
	\begin{definition}[pseudo-dimension]
		Given a model class $\CH$ let $\CH^{+}$\! contain the subgraphs of functions in $\CH$, that is
		$\CH^{+} \,=\,\{\, h^{+}\;|\;h \in \CH\,\}$ where $h^{+}\,\doteq\, \{\hspace{0.3mm}(x,t) \in \RR^d \times \RR\;|\; t \leq h(x)\hspace{0.3mm}\},$
		then, the VC dimension of class $\CH^{+}$, that is $V_{\CH^{+}}$, is called the pseudo-dimension of $\CH$.
	\end{definition}
	
	\subsection{Empirical Risk Minimization-Based Ranking}\label{subsec:ERM-based}
	
	For a given parameter 
	$\theta$,
        let us consider the ERM estimate for all datasets, i.e., let
	\begin{equation}
		f_{\theta,n}^{(j)} \,\doteq\,\argmin_{f \in \CF_0} \frac{1}{n}\sum_{i=1}^n (f(X_i)-Y_{i,j}(\theta))^2,
	\end{equation}
	for $j =0,\dots,m-1$.
	With these models let the reference variables be defined as
	\begin{equation}\label{def:reference-var}
		Z_n^{(j)}(\theta) \,\doteq\, \frac{1}{n}\sum_{i=1}^n (f_\theta(X_i)- {f}_{\theta,n}^{(j)}(X_i))^2,
	\end{equation}
	for $j=0,\dots,m-1$, that is, as the empirical error terms w.r.t.\ candidate function $f_\theta$ we are currently testing. Finally, let us use the ranking function defined in \eqref{eq:ranking-function}.
	
	Notice that $f_{\theta,n}^{(0)}$ does not depend on parameter $\theta$ because it corresponds to the original dataset, $\CD_0$; therefore let us denote it by $f_{*,n}^{(0)}$. Then, on the one hand, $f_{*,n}^{(0)}$ should converge to $f^*$ for all $\theta \in \Theta$, because it only uses the original sample, but on the other hand, $f_{\theta,n}^{(j)}$ should tend to $f_\theta$ for all $j \in [m-1]$, for  the reason that they estimate the regression function based on samples generated from the conditional distribution determined by $f_\theta$. Because of these observations, our intuition is that for $\theta \neq \theta^*$ for $n$ large enough $Z_n^{(0)}(\theta)$ tends to be the greatest. Therefore, we choose $q_1=1$ and $q_2 = q$ leading to the confidence region
	\begin{equation}
		\label{eq:confidence-region}
		\Theta_{n}^{(1)} \,\doteq\, \big\{ \, \theta \in \Theta \;|\;  \psi\big(\,\CD^{\pi}_0, \{ \CD^{\pi}_j(\theta) \}_{j \neq 0}\,\big) \leq q\, \big\}.
		\vspace{1mm}
	\end{equation}
	For proving strong uniform consistency we make two additional assumptions:
        \begin{assumption}\label{ass:b5}
            (inverse Lipschitz condition) Let $(\Theta,\Delta)$ be a metric space and let the parameterization be inverse Lipschitz-continuous, i.e., there exists a constant real number $L>0$ such that for all $\theta_1,\theta_2 \in \Theta$ we have
				$\Delta(\theta_1,\theta_2) \,\leq\, L \cdot \|\, f_{\theta_1} - f_{\theta_2} \|_{\scriptscriptstyle Q}.$
        \end{assumption}
        \begin{assumption}\label{ass:b6}
            (finite pseudo-dimension) $V_{\CF_0^{+}}<\infty.$
        \end{assumption}

        \begin{theorem}\label{theorem:algorithm-I}
		Assume A\ref{ass:b1}-A\ref{ass:b6}, then for all sample size $n \in \NN$ and integers $q\leq m$ we have $\PP\big(\,\theta^* \in \Theta_{n}^{(1)}\,\big)\, = \, \nicefrac{q}{m}$.
		  In addition, if\, $q <m$, then $(\Theta_{n }^{(1)})_{n \in \NN}$ is strongly uniformly consistent, i.e., by 
        denoting $B(\theta^*, \varepsilon) \doteq \{\hspace{0.3mm}\theta \in \Theta\;|\;\Delta(\theta,\theta^*)< \varepsilon\hspace{0.3mm}\}$, we have for 
        all $\varepsilon > 0$ that
         \begin{equation}\label{eq:strong-uniform-consistency}
         \PP \bigg( \bigcup_{n=1}^\infty \bigcap_{k=n}^\infty \big\{ \widehat{\Theta}_k^{(1)} \subseteq B(\theta^*, \varepsilon)\big\}  \bigg) = 1.
         \end{equation}
\end{theorem}
	
	Observe that Theorem \ref{theorem:algorithm-I} provides uniform stochastic guarantees, i.e., for every $\varepsilon > 0$ with probability $1$ the confidence regions are included in $B(\theta^*, \varepsilon)$ with at most finitely many exceptions. We made only very mild statistical and structural assumptions. The function class can be rich, e.g., it can be a subset of an infinite dimensional vector space; only the pseudo-dimension of the model class needs to be bounded, therefore the distribution of the sample is mildly restricted. The inverse Lipschitz condition is required intuitively to ensure that parameters, whose models are similar in the $\mathcal{L}_2(\QX)$ sense, should be close to each other.
     An important observation is that the confidence region is built around the ERM estimator, because $f_{\hat{\theta}} = f_{\hat{\theta},n}^{(0)}$ and $Z_n^{(0)}(\hat{\theta})=0$, henceforth the rank equals to $1$ for $\hat{\theta}$ with high probability.

	A quantitative version of Theorem \ref{theorem:algorithm-I} is also formulated to provide PAC guarantees. The proofs, including the constants involved, are presented in Appendix \ref{app:algorithm-I}.
	\begin{theorem}\label{theorem:quantitative-uniform}
		Assume A\ref{ass:b1}-A\ref{ass:b6} and $q <m$. For all $\varepsilon >0$ there exists $N_0$ such that for $n \geq N_0$:
		\begin{equation}\label{eq:PAC-uniform}
			\PP \,\big( \,\Theta_{n}^{(1)} \subseteq B(\theta^*, \varepsilon)\,\big) \,\geq\, 1- C_1\,e^{-n\lambda_1},
		\end{equation}
		where $C_1$ depends only on $\varepsilon$ and  $V_{\CF_0^{+}}$, and $\lambda_1$ depends only on $\varepsilon$ and the Lipschitz constant.
	\end{theorem}

    \subsection{Example: Perceptron and Generalized Linear Models}
	
	An important ERM estimator is the perceptron \citep{rosenblatt1958perceptron}, which is closely related to generalized linear models in statistics \citep{hastie2017generalized}. They apply an activation (or link) function on an affine transformation of the input vectors, i.e.,
    $f_{(\theta,b)}(x) \doteq \sigma(\theta\tr x + b)$ for $\theta, x \in \RR^d$ and $b \in \RR$.
    They can also be interpreted as simple neural networks with a single artificial neuron. 
    The classical approach of Rosenblatt applies the sign function as activation and its iterative learning-rate-free perceptron algorithm terminates in finite steps, for linearly separable data \citep{understandingML}.
    
    For potentially non-linearly separable data and general activation functions, one can apply a least-squares type approach. This however, typically leads to non-convex optimization problems.
    Approximate numerical solutions can be achieved by stochastic gradient descent type methods, for example, by the celebrated backpropagation algorithm \citep{lecun1988theoretical}.
	
    Now, let us use $\sigma(z) \,\doteq\, 2/(1+ e^{-z})-1$. We want to estimate the regression function of binary classification, in case of $\{-1,1\}$ valued classes. Let us consider $\CF_0$, the model class of 
	$\{f_{ ( \theta ,b)}\}$ functions with
    $(\theta\tr,b)\tr \in \RR^{d+1}$. One can show that the pseudo-dimension of neural networks with fixed structure is finite        \citep{anthony2009neural}, thus, as a special case, $V_{\CF_0^{+}} < \infty$, implying that, by Theorem          \ref{theorem:algorithm-II}, our method constructs non\hyp asymptotically exact confidence regions in $\CF_0$, 
    which are strongly uniformly consistent in $\mathcal{L}_2(\QX)$.

    The perceptron with the logistic type activation function above differs from the logistic regression estimator in the used error criterion only. The perceptron typically uses the least-squares estimator, whereas for logistic regression one needs to maximize the conditional likelihood function.
    In Section \ref{sec:numerical-tests} we present numerical experiments to compare the corresponding confidence regions based on the presented resampling framework.

    In Appendix \ref{app:model-classes} additional model classes, namely linear models, feed-forward neural networks and radial basis function networks, are considered as demonstrative examples. In these use cases the ERM-based method constructs
    confidence regions for the true regression function for every finite sample size, all of which are strongly uniformly consistent.

    \section{Local Averaging Estimators}
	
    Our resampling framework allows the use of any tractable regression function estimator. As another example, in this section we consider the kNN-based construction originally suggested by \cite{csajitamas2019}.
    We strengthen \cite[Theorem 2]{csajitamas2019} and prove strong pointwise consistency under milder distributional assumptions. Particularly, in Theorem \ref{theorem:algorithm-II} marginal distribution $\QX$ does not need to be absolutely continuous and the support of $\QX$ does not need to be compact. Further, we present Theorem \ref{theorem:quantitative-knn}, a quantitative result
    providing an exponential bound on the exclusion probability of false models.
	
	\subsection{The kNN Estimator}
	
	Let $\BX \subseteq \RR^d$ as above with the Euclidean metric. The kNN estimator is defined by 
	\vspace{-0.5mm}
	\begin{equation}
		{f}_n(x)\, \doteq\, \frac{1}{k_n} \sum_{i=1}^n Y_i\, \BI( X_i \in N_n(x,k_n)),
		\vspace{-0.5mm}
	\end{equation}
        where $\{(X_i,Y_i)\}$ is an i.i.d.\ sample and $N_n(x,k_n)$ denotes the set of $k_n$ closest points in $\{X_i\}_{i=1}^n$ to $x$. We assume that $\norm{X-X'}=\norm{X-X''}$ holds with probability zero for $i\neq j \in [n]$, where $\{X, \hspace{0.2mm}X'\!, X''\}$ are i.i.d. copies. This assumption is needed to define the $k_n$ closest neighors in $\{X_i\}_{i=1}^n$ $\QX$-almost surely. Note that this assumption can be eliminated by using random tie breaking for those $x \in \BX$ and $\{X_i, X_j\}$ for which we have $\norm{X_i - x} = \norm{X_j - x}$.
	By \cite[Theorem 23.7]{gyorfi2002distribution} if $k_n/n \to 0$ and $k_n \to \infty$ as $n\to \infty$, then 
	\vspace{-0.5mm}
	\begin{equation}
		\int \big(\,{f}_n(x)- f^*(x)\,\big)^2 \dd \QX(x)\, \xrightarrow{\,a.s.\,}\,0. 
		\vspace{-0.5mm}
	\end{equation}

	\subsection{Neighbor-Based Ranking}
	
	We follow the procedure in Section \ref{subsec:ERM-based} to generate alternative samples for a given candidate model. Then, we consider the kNN estimates for every dataset, i.e., we let
	\begin{equation}
		{f}_{\theta,n}^{(j)}(x)\, \doteq\, \frac{1}{k_n} \sum_{i=1}^n Y_{i,j}(\theta)\, \BI( X_i \in N_n(x,k_n)),
	\end{equation}
	for all $j =0,\dots,m-1$. Finally, we define the reference variables as in \eqref{def:reference-var}, the ranking function as in \eqref{eq:ranking-function} and the confidence region for integers $q$ and $m$ by
	\begin{equation}
		\label{eq:confidence-region2}
		\Theta_{n}^{(2)} \,\doteq\, \big\{ \, \theta \in \Theta \;|\;  \psi\big(\,\CD^{\pi}_0, \{ \CD^{\pi}_j(\theta) \}_{j \neq 0}\,\big) \leq q\, \big\}.
	\end{equation}
	The guarantees of the construction are summarized as:
	\begin{theorem}\label{theorem:algorithm-II}
		Assume A\ref{ass:b1}-A\ref{ass:b4}, then for all sample size $n \in \NN$ and integers $q\leq m$ we have $\PP\big(\,\theta^* \in \Theta_{\varrho,n }^{(2)}\,\big) = \nicefrac{q}{m}$.
	    Further, if $q <m$, 
		$\nicefrac{k_n}{n} \to 0$, and there exists $\delta >0$ such that $\nicefrac{k_n^2}{n^{(1+\delta)}} \to \infty$ as $n\to\infty$, then $(\,\Theta_{n }^{(2)}\,)_{n \in \NN}$ is strongly pointwise consistent, i.e., for $\theta\neq\theta^*$:
        \begin{equation}
            \PP\bigg( \bigcap_{n=1}^\infty \bigcup_{k=n}^\infty \big\{ \theta \in \widehat{\Theta}_k^{(2)}\big\} \bigg) = 0
        \end{equation}
            holds for all distribution of $(X,Y)$ such that
		$\PP\big(\hspace{0.3mm}\norm{X- X'}_2 \neq \norm{X-X''}_2\hspace{0.3mm}\big) =1,$
		where $X'$ and $X''$ are independent copies of $X$.
	\end{theorem}

	The proof can be found in Appendix \ref{app:algorithm-II}. Observe that the statement is distribution-free as $\CF_0$ may contain all possible regression functions and ties can be resolved by a random permutation. As before, the exact confidence level is non-asymptotically guaranteed and easily adjustable. Pointwise consistency ensures that false parameters are included infinitely many times in the confidence regions with probability $0$. The conditions on $k_n$ can be easily satisfied, e.g., $k_n=\floor*{n^\alpha}$ for $\alpha \in (\nicefrac{1}{2},1)$ are appropriate choices. In the quantitative version of Theorem \ref{theorem:algorithm-II} we prove an exponential bound on the exclusion probability for any $\theta \neq \theta^*$\!. 
	\begin{theorem}\label{theorem:quantitative-knn}
		Assume that A\ref{ass:b1}-A\ref{ass:b4} and $q <m$ hold. In addition, $\nicefrac{k_n}{n} \to 0$, and there exists $\delta >0$ such that $\nicefrac{k_n^2}{n^{(1+\delta)}} \to \infty$ as $n\to\infty$. For every distribution of $(X,Y)$ such that\,
		$\PP\big(\norm{X- X'}_2 \neq \norm{X-X''}_2\big) =1,$
		where $X'$ and $X''$ are independent copies of $X$, for a given $\theta \in \Theta$, $\theta\neq \theta^*$ there exists $N_0$ such that for $n \geq N_0$  we have
		\begin{equation}\label{eq:PAC-pointwise}
			\PP\big(\, \theta \notin \Theta_{n}^{(2)}\,\big)\,\geq\, 1 -C_2\, e^{-\nicefrac{k_n^2\lambda_2}{n}},
		\end{equation}
		where $C_2$ depends only on $m$, the number of datasets, and $\lambda_2$ depends on $\kappa \doteq \|f_\theta - f^*\|_{\scriptscriptstyle Q}$.
	\end{theorem}

	In Theorem \ref{theorem:quantitative-knn} the required number of data points is not quantified explicitly for the exponential bound. In the appendix it is argued that $n$ should be large enough to reduce the expected losses below $\nicefrac{\kappa}{8}$, which is ensured by the weak consistency of kNN estimators, by Theorem \ref{thm:weak-consistency}.  However, this weak consistency can be arbitrarily slow, hence, the sample size that is needed for \eqref{eq:PAC-pointwise} depends on the distribution. A possible bound on the minimum sample size can be found via rate of convergence results, which restrict the model class (e.g., by requiring Lipschitzness) and make distributional assumptions (e.g., limited conditional variance function, bounded support). For such results we refer to \cite[Theorem 14.5]{biau2015lectures}, \cite[Theorem 6]{doring2017rate} and \cite[Theorem 6.2]{gyorfi2002distribution}.

    \section{Numerical Experiments}\label{sec:numerical-tests}
		
	We carried out numerical experiments on synthetic datasets to demonstrate the suggested algorithms. In this example we considered a logistic model class defined as
	\begin{equation}\label{eq:logistic}
		\CF_0 \doteq \Big\{\hspace{0.3mm}f_{(a,b)}(x)\doteq 2/(1 + e^{-(b\cdot x +a)}) -1 \;\big|\; a, b \in \RR\hspace{0.3mm}\Big\}
	\end{equation}
	and compared the perceptron-based and the kNN-based method to the asymptotic normal confidence ellipsoids around the MLE, whose formulation is included in Appendix \ref{app:mle}.  
    Output variable $Y$ took values $-1$ and $1$ with probability $\nicefrac{1}{2}$ each. 
    The conditional distribution of $X$ on $Y$ was Gaussian with mean $Y$ and variance $1$, in which case $f^*(x) = 2/(1- e^{-2x})-1$ and $[a_*,b_*] = [0,2]$, thus the true parameter (vector) was included in $\CF_0$. The sample size was set to $500$ and we used $m=40$ as the resampling parameter. We tested the model parameters on a fine grid and plotted their relative ranks, which indicate the rejection probabilities of the test, see the color bar in Figure \ref{fig:knn}. The results of the perceptron-based construction and the kNN-based construction are presented in Figure \ref{fig:perceptron} and Figure \ref{fig:knn}. The standard asymptotic confidence ellipsoids centered around the MLE are shown in Figure \ref{fig:Asymptotic}. 
	
	Recall that our resampling framework can use any specific regression function estimator, therefore we also tested an MLE-based algorithm, for which the reference variables, \eqref{def:reference-var},  are defined with the help of the MLE for all parameters.  Figure \ref{fig:ML} shows the results for this MLE-based resampling method. By Theorem \ref{theorem:exact-confidence} this variant constructs exact confidence regions for any finite sample similarly to the kNN-based and ERM-based versions. 
	
    The plots illustrate that our resampling framework produces comparable region estimates to the asymptotic ellipsoids. It is clear that those methods which use a priori knowledge about the regression function's parametric structure construct tighter bounds, if the a priori information is correct.
    We also plotted the point estimates and the true parameters on the figures. Note that the kNN estimate was not included      in $\CF_0$, therefore it is not marked.

\begin{figure*}[!t]
    \centering
		\begin{subfloat}[Perceptron-Based Levels\label{fig:perceptron}]{
        \centering
        \includegraphics[width=2.6in]{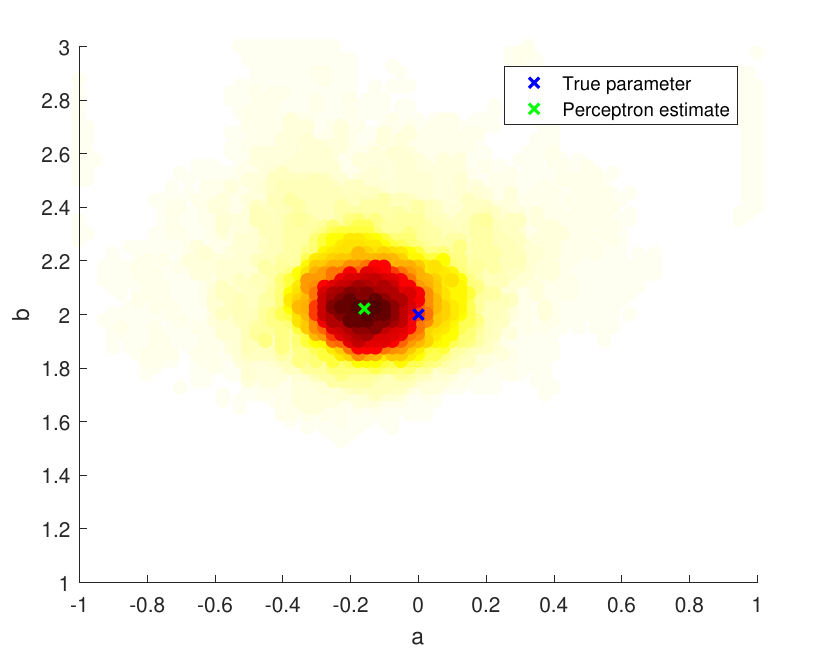}}
            \end{subfloat}
            \quad
		\begin{subfloat}[kNN-Based Levels\label{fig:knn}]{\centering
        \includegraphics[width=2.6in]{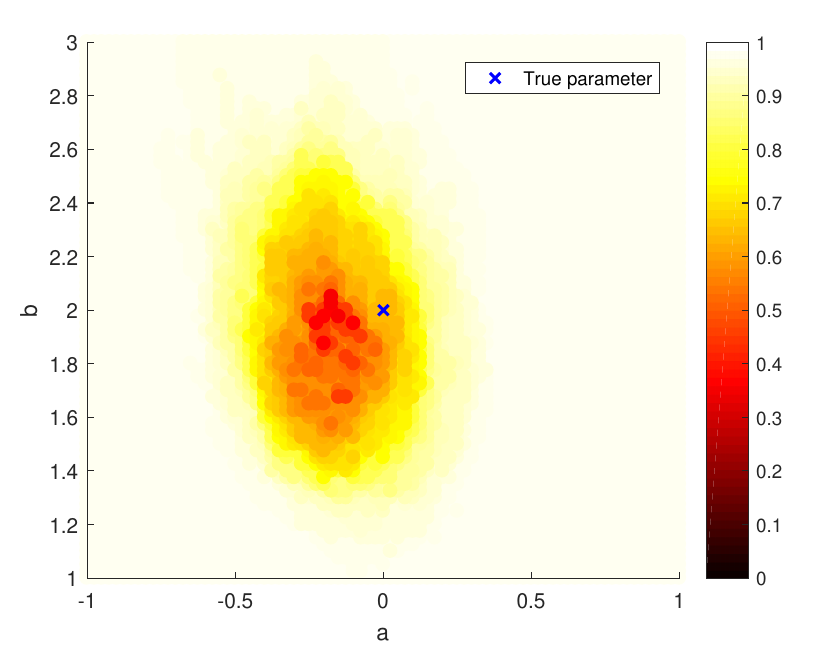}}
			       \end{subfloat}
		\\
           \hspace{3mm}
        \begin{subfloat}[Asymptotic Ellipsoids\label{fig:Asymptotic}]{\centering
        \includegraphics[width=2.6in]{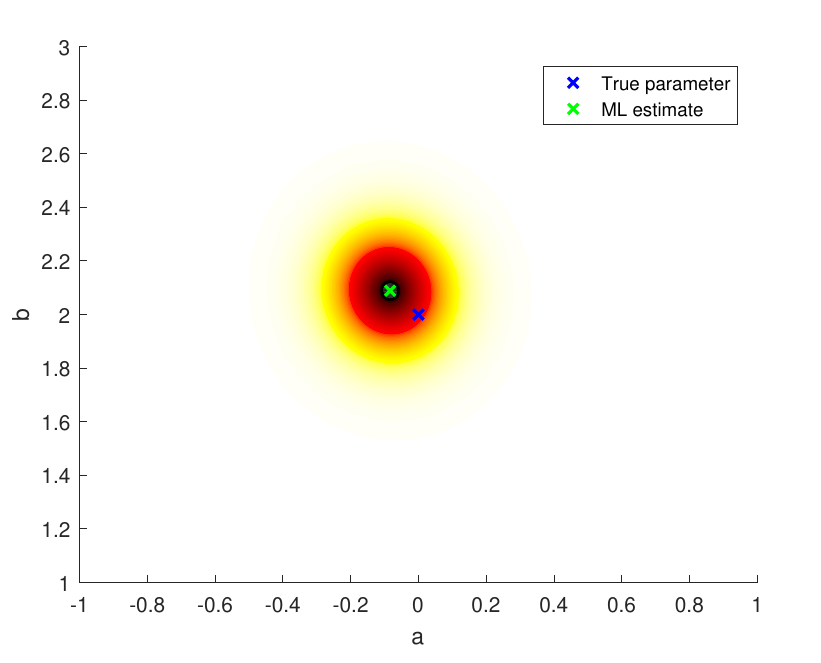}}
            \end{subfloat}
           	\quad
		\begin{subfloat}[MLE-Based Levels\label{fig:ML}]{\centering
        \includegraphics[width=2.6in]{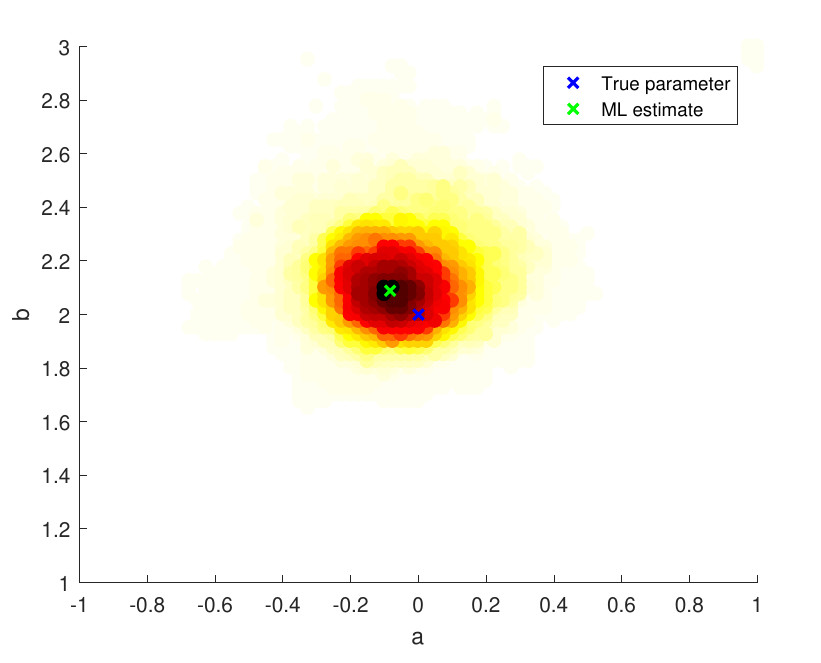}}
	    \end{subfloat}	
        \medskip
        \caption{\small Families of confidence regions for the perceptron-based, the kNN-based and the MLE-based resampling constuctions are presented along with the asymptotic confidence ellipsoids. 
        We tested a class of logistic functions parameterized by $a$ and $b$, see \eqref{eq:logistic}.
        The relative ranks, which indicate the rejection probabilities, are plotted with their colors for each parameter pair.
        }
		\label{fig_sim}
	\end{figure*}

    We estimated the coverage probability of the  regions with theoretical value $95\%$ based on $30\,000$ Monte Carlo experiments, which is by Hoeffding's inequality sufficient to reach a $1\%$ precision with probability at least $99 \%$. We tested if the true parameter was included in the confidence region with hyperparameters $m=20$ and $q=1$, and estimated the confidence level by the frequency of inclusions. 
	The kNN-based, the MLE-based and the perceptron-based constructions were all applied. Similarly, we tested if the true parameter was included in the asymptotic confidence ellipsoid with approximate confidence level $95\%$ and calculated the frequency of inclusions. We carried out tests for $n=20$, $50$ and $100$ observations. We can see in Table \ref{conf_level} that the confidence levels of the asymptotic regions are conservative and higher than $96 \%$ for $n=20$ and $n=50$  instead of the desired $95 \%$, while our exact methods are close to the theoretical value for every sample size. 
    	
    We carried out another experiment, when the marginal distribution of the input was uniform on $[-1,1]$, but the regression function and the model class remained the same.  These results, shown in Table \ref{conf_level2}, validate that our resampling framework provides exact guarantees for any sample size, for all variants, as before. On the other hand, the coverage levels of the asymptotic confidence sets (column ``Ellipsoid'') are inaccurate for small sample sizes, though they tend to the theoretical value ($95 \%$) as the sample size $n$ increases.

    We can conclude that our non-asymptotic exact guarantee is preferable to the asymptotic one, since the latter is conservative for small samples, even if its a priori assumptions hold. 

     {\renewcommand{\arraystretch}{1.25}
    	\begin{table}[t]
			\begin{minipage}{.47\linewidth}
                \scriptsize
				\caption{\small Empirical coverage probabilities for $n=20$, $50$ and $100$ observations estimated from $30\,000$ Monte Carlo trials. The underlying conditional marginal distribution of $X$ given $Y$ was normal with mean $Y$ and variance $1$.\label{conf_level}\vspace*{1mm}}
                \begin{tabular}{p{0.4cm} p{1.2cm} p{1cm} p{1cm} p{1.1cm}} 
                    \toprule
                    $n$ & Ellipsoid & kNN & MLE & Perceptron\Tstrut\Bstrut\\ 
                    \midrule
                    20 & 96.50 $\%$ & 95.21 $\%$ & 95.10 $\%$ & 95.19 $\%$ \Tstrut\Bstrut\\
                    50 & 96.22 $\%$ & 94.84 $\%$ & 95.07 $\%$ & 94.90 $\%$ \Tstrut\Bstrut\\
                    100 & 95.86 $\%$ & 95.01 $\%$ & 94.88 $\%$ & 94.83 $\%$\Tstrut\Bstrut \\
                    \midrule
                \end{tabular}                
			\end{minipage}
			\hspace{3mm}
			\begin{minipage}{.47\linewidth}
                \scriptsize
				\caption{\small Empirical coverage probabilities for $n=20$, $50$ and $100$ observations estimated from $30\,000$ Monte Carlo trials.  The marginal distribution of the inputs was uniform on $[-1,1]$. The regression function was the same.   
                \label{conf_level2}\vspace*{1mm}}
                \begin{tabular}{p{0.4cm} p{1.2cm} p{1cm} p{1cm} p{1.1cm}} 
                    \toprule
                    $n$ & Ellipsoid & kNN & MLE & Perceptron\Tstrut\Bstrut\\ 
                    \midrule
                    20 & 97.49 $\%$ & 94.92 $\%$ & 95.01 $\%$ & 95.09 $\%$ \Tstrut\Bstrut\\
                    50 & 96.63 $\%$ & 94.99 $\%$ & 94.90 $\%$ & 94.69 $\%$ \Tstrut\Bstrut\\
                    100 & 95.63 $\%$ & 94.72 $\%$ & 94.89 $\%$ & 94.94 $\%$\Tstrut\Bstrut \\
                    \midrule
                \end{tabular}                
			\end{minipage}
            \vspace*{-3mm}
		\end{table}

          \section{Discussion}

         In this paper we studied the problem of  assessing the uncertainty of regression function estimates for binary classification in the form of confidence regions. Such regions are of high practical importance, as they are essential, e.g., for risk management and robust decision making. The regression function is one of the key objects of this problem, because it not only determines an optimal classifier (if we work with the zero-one loss), but it can also be used to compute the probability of misclassification for any given input. We extended and improved the distribution-free framework originally suggested in \citep{csajitamas2019}. Our data-driven resampling approach provides exact stochastic guarantees for finding the true underlying regression function under the i.i.d.\ assumption for any finite sample size. One of the key ideas is to test each candidate by generating alternative samples by a suitable resampling mechanism, and then compare those to the original dataset using a rank test.
         The presented algorithm can exactly reach any (rational) confidence level. We investigated two specific methods, as well. The ERM-based construction was introduced in this paper, while the kNN-based method was suggested earlier in the aforementioned paper. We proved the strong consistency of both methods under very mild statistical assumptions. In the kNN-case, we significantly relaxed the assumptions of the original paper. For our ERM-based method the consistency is uniform, while for kNN-based version it is pointwise. We also proved exponential PAC bounds on the probabiliy of excluding false parameters for both approaches. Finally, we demonstrated the framework by numerical experiments and compared the results to the confidence ellipsoids of the MLE method for logistic regression.


\begin{appendix}
\section{Proof of Lemma \ref{lemma:uniform-rank-distribution}}\label{app:discrete-uniform}

\begin{customlemma}{\ref{lemma:uniform-rank-distribution}}
Let $\xi_1, \xi_2, \dots, \xi_m$ be 
    almost surely pairwise different exchangeable random elements taking values in an arbitrary measurable space, and 
    let  $\psi$ be a ranking function. Then, the rank, $\psi(\xi_1, \dots, \xi_m)$, is distributed uniformly on $[\hspace{0.3mm}m\hspace{0.3mm}]$.
\end{customlemma}

\begin{proof}
Random elements $\{A_i\}_{i=1}^m$ are exchangeable, thus
\begin{equation}
\begin{aligned}
&\PP \big(\,\psi\big(\,A_{1}, \dots, A_{m}\,\big)\,= \, k\,\big) = \;\PP \big(\,\psi\big(\,A_{\mu(1)}, \dots, A_{\mu(m)}\,\big)\, = \, k\,\big)
\end{aligned}
\end{equation}
for all $k\in [\,m\,]$ and permutation $\mu$ on set $[\,m\,]$.\ Let us fix the value of $k$. Observe that because of P\ref{ass:p1} we have
\begin{equation}
\begin{aligned}
&\big\{ \,\psi\big(\,A_{1}, \dots, A_{m}\big)= k\,\big\}=\big\{ \,\psi\big(\,A_{1}, A_{\sigma(2)} \dots, A_{\sigma(m)}\big)=  k\,\big\} 
\end{aligned}
\end{equation}
for all permutation $\sigma$ on set $\{2, \dots, m \}$. Using the same notation as in P\ref{ass:p2} let
\begin{equation}
    C_i \doteq \big\{ \psi\big(\,A_{i}, \{ A_{j} \}_{j \neq i},\big)=  k\,\big\}.
\end{equation} Notice that $\big\{ C_i \big\}_{i=1}^m$ are disjoint because of P\ref{ass:p2}. We show that they cover a probability one event. Let $\Omega_0$ be the zero probability event which covers those cases when there exists $(i,j) \in [m]\times [m]$ such that $A_i = A_j$. Because of P\ref{ass:p2} for all $\omega \in \Omega \setminus \Omega_0$ the elements of $\{\psi (A_j(\omega), \{A_k(\omega)\}_{k \neq j})\}_{j=1}^m$ are different. By definition these values are included in $[m]$,  hence there exists $i \in [m]$ such that $\psi (A_i(\omega), \{A_k(\omega)\}_{k \neq i})=k$ implying that $\omega \in \textstyle \bigcup_{i=1}^m C_i$ and $\Omega \setminus \Omega_0 \subseteq \textstyle \bigcup_{i=1}^m C_i$. Consequently  $\PP \Big(\textstyle \bigcup_{i=1}^m C_i\Big) =1$. Putting this together with the disjoint property of events $\{C_i \}_{i=1}^m$, and the exchangeability of $\{A_i \}_{i=1}^m$ yields
\begin{equation}
\begin{aligned}
1 &= \PP \big(\Omega \setminus \Omega_0 \big) = \PP \big( \textstyle \bigcup_{i=1}^m C_i \big)= \sum_{i=1}^m \PP\big(C_i\big) = \sum_{i=1}^m \PP\big( \, \psi\big(\,A_{i}, \{ A_{j} \}_{j \neq i},\big)=  k\,\big)\\
&= \sum_{i=1}^m \PP( \psi\big(\,A_{1}, \{ A_{j} \}_{j \neq 1},\big)=  k\,) = m\, \PP\big(  \,\psi\big(\,A_{1}, \dots, A_{m}\big)= k\,\big),
\end{aligned}
\end{equation}
from which it follows that $\PP( \,\psi\big(\,A_{1}, \dots, A_{m}\big)= k\,) = \nicefrac{1}{m}$ for all $k \in [\,m\,]$.
\end{proof}

\section{Proof of Theorem \ref{theorem:exact-confidence}}

\begin{customthm}{\ref{theorem:exact-confidence}}
		Assume that A\ref{ass:b1}, A\ref{ass:b2} and A\ref{ass:b3} hold. Then, for any ranking function $\psi$ and integers $1\, \leq\, q_1\, \leq\, q_2\, \leq\, m$, we have
		$\PP\big(\, \theta^* \in \Theta_{n}^{\psi}  \, \big)\; = \; (q_2-q_1+1)/m.$
\end{customthm} 
	\begin{proof}
		We observed that $\CD_0$, $\CD_1(\theta^*), \dots , \CD_{m-1}(\theta^*)$ are conditionally i.i.d.\ with respect to the inputs, $\{X_i\}_{i=1}^n$, hence their extended versions are exchangeable, and also (a.s.) pairwise different due to their constructions. 
		Thus, Lemma \ref{lemma:uniform-rank-distribution} implies that $\psi ( \CD_0^{\pi}(\theta),  \{ \CD^{\pi}_k(\theta) \}_{k \neq 0} ) = k$ with exact probability $\nicefrac{1}{m}$ for all $k \in [m]$.
	\end{proof}

\section{Proofs of Theorem \ref{theorem:algorithm-I} and Theorem \ref{theorem:quantitative-uniform}}\label{app:algorithm-I}

The proof of strong uniform consistency relies on the ULLN, in particular on the concepts of packing and covering numbers and three auxiliary lemmas. First, we present these concepts and three lemmas. Then, we apply the lemmas to prove the theorems. 
	
	Condition A\ref{ass:b6} bounds the pseudo-dimension of $\CF_0$ which is sufficient for the ULLN to hold on $\CF_0$ \cite[Theorem 9.6]{gyorfi2002distribution}, however, we need to apply the ULLN on the transformation of $\CF_0$, i.e., let us consider the set of $\BX \to \RR$ type functions given by
	\begin{equation}\label{eq:h-definition}
		\CG\,\doteq\, \{\,h=(f-g)^2 \;\,|\,\; f,g \in \CF_0\,\}.
	\end{equation}
	The first step in the proof is to show that condition A\ref{ass:b6} is sufficient for $V_{\CG^{+}}<\infty$ to hold. Proving this is straightforward with the help of covering and packing numbers. Our reasoning is motivated by the arguments in \citep{gyorfi2002distribution, understandingML, Vapnik1998, wainwright2019high}.
	\begin{definition}[$\varepsilon$-covering number and packing number]
		Let $\CF$ be a class of real-valued functions and $\norm{\cdot}$ be a norm on $\CF$ and $\varepsilon >0$.
		\begin{enumerate}[leftmargin = 0.6cm]
			\item
			A set $\{ f_1, \dots, f_k \}$ is an $\varepsilon$-cover of $\CF$ w.r.t.\ $\norm{\cdot}$, if for all $f \in \CF$ there exists an index $j \in [k]$ such that $\norm{f_j-f}\, \leq\, \varepsilon$.
			The size of the smallest $\varepsilon$-cover of $\CF$ w.r.t.\ $\norm{\cdot}$ is denoted by $\mathcal{N}(\varepsilon, \CF,\norm{\cdot})$ and it is called the $\varepsilon$-covering number of $\CF$ w.r.t.\ $\norm{\cdot}$.
			\item
			A set $\{ f_1, \dots, f_l \} \subseteq  \CF$ is an $\varepsilon$-packing of $\CF$ w.r.t.\ $\norm{\cdot}$ if for all $1\leq i < j \leq l$ we have
			$\norm{f_i-f_j}\, >\, \varepsilon$.
			The size of the largest $\varepsilon$-packing of $\CF$ w.r.t.\ $\norm{\cdot}$ is denoted by $\mathcal{M}(\varepsilon, \CF,\norm{\cdot})$ and it is called the $\varepsilon$-packing number of $\CF$ w.r.t.\ $\norm{\cdot}$.
		\end{enumerate}
            Both $\mathcal{N}(\varepsilon, \CF,\norm{\cdot})$ and $\mathcal{M}(\varepsilon, \CF,\norm{\cdot})$ can be infinite.
	\end{definition}
	In the proof we consider random $\mathcal{L}_p$ covers with respect to the empirical measure, that is for a random vector $(X_1,\dots,X_n) \in \BX^n$ let $\mathbb{P}_n$ denote the empirical measure, i.e., let $\mathbb{P}_n(A) \doteq \frac{1}{n}\sum_{i=1}^n \BI\,(X_i \in A)$ for all measurable $A$ and let
	\begin{equation}
		\begin{aligned}
		&\CN_1(\,\varepsilon,\CG,\{X_1,\dots,X_n\}\,) \,\doteq\, \CN \big( \,\varepsilon,\CG, \norm{\cdot}_{\mathcal{L}_1(\mathbb{P}_n)}\big), \\
            &\mathcal{M}_1(\,\varepsilon,\CG,\{X_1,\dots,X_n\}\,) \,\doteq\, \mathcal{M} \big( \,\varepsilon,\CG, \norm{\cdot}_{\mathcal{L}_1(\mathbb{P}_n)}\big),
		\end{aligned}
	\end{equation}
	denote the random $\mathcal{L}_1(\mathbb{P}_n)$ $\varepsilon$-covering and packing numbers of $\CG$. It is known that a sufficient condition for the ULLN can be formulated with the help of these terms, i.e., the ULLN holds if the expectations of the $\mathcal{L}_1(\mathbb{P}_n)$-covering numbers are summable \cite[Theorem 9.1]{gyorfi2002distribution}. The referred theorem is included in Appendix \ref{app:theorems}. Thus, bounding the expectation of the $\mathcal{L}_1(\mathbb{P}_n)$-covering number is a key step in the proof of Theorem \ref{theorem:algorithm-I}.

	\begin{lemma}\label{lemma:covering-bound}
		Let $X_1, \dots, X_n$ be random vectors in $\RR^d$. 
		If A\ref{ass:b6} holds, then for all $\varepsilon>0$ there exists a constant $C$, independent of $X_1,\dots X_n$, such that
		\begin{equation}\label{eq:bounded-covering}
			\CN_1(\,\varepsilon,\CG,\{X_1,\dots,X_n\}\,) \,\leq\, C.
		\end{equation}
	\end{lemma}
	\begin{proof}
First, we show that for any $x_1,\dots,x_n \in \BX$ we have
\begin{equation}\label{eq:l1covers}
\begin{aligned}
&\CN_1(\,\varepsilon,\CG,\{x_1,\dots,x_n\}\,) \leq \big(\CN_1(\varepsilon/8,\CF_0,\{x_1,\dots,x_n\})\big)^2.
\end{aligned}
\end{equation}
Let $\{f_1,\dots,f_k\}$ be an $\nicefrac{\varepsilon}{8}$-cover of $\CF_0$ w.r.t. the $\mathcal{L}_1(\mathbb{P}_n)$ norm with minimal size. It can be assumed w.l.o.g. that $f_i$ is bounded between $[-1,1]$ for all $i \in [k]$, otherwise the truncated versions can be used. We argue that $\{\,(f_i-f_j)^2\,|\,i,j \in [k]\,\}$ is an $\varepsilon$-cover of $\CG$ including at most $k^2$ elements. For an arbitrary $h \in \CG$ there exist $f,g \in \CF_0$ such that $h=(f-g)^2$. For these functions one can find indexes $i(f), j(g) \in [k]$ such that
\begin{equation}\label{eq:cover}
\begin{aligned}
&\norm{f-f_i}_{\mathcal{L}_1(\mathbb{P}_n)} \leq \nicefrac{\varepsilon}{8} \quad \text{and}\quad \norm{g-f_j}_{\mathcal{L}_1(\mathbb{P}_n)} \leq \nicefrac{\varepsilon}{8}.
\end{aligned} 
\end{equation}
Since $f,g,f_i$ and $f_j$ are all bounded in $[-1,1]$ and \eqref{eq:cover} holds, we obtain that
\begin{equation}
\begin{aligned}
\int \big| h- (f_i -f_j)^2\big| \textrm{d} \mathbb{P}_n &= \int \big| (f-g)^2- (f_i -f_j)^2\big| \textrm{d} \mathbb{P}_n \\
&\leq \int \,\big|(f-g + f_i -f_j)(f-g -f_i +f_j)\,\big|\,\textrm{d} \mathbb{P}_n \\
&\leq 4 \int \,\big|f-f_i\big|\,\textrm{d} \mathbb{P}_n + 4 \int \,\big| g-f_j\big|\,\textrm{d} \mathbb{P}_n \leq \varepsilon,
\end{aligned}
\end{equation}
thus \eqref{eq:l1covers} is proved. In the second step we consider function class $\widetilde{\CF} \doteq \{\,\nicefrac{1}{2}(f+1)\,|\,f\in \CF_0\,\}$ to satisfy the condition of Theorem \ref{thm:vcbound-expect}, requiring the boundedness of each $\tilde{f} \in \widetilde{\CF}$ in $[0,1]$. Notice that this transformation does not affect the pseudo-dimension, that is  $V_{\CF_0^{+}} = V_{\widetilde{\CF}^{+}}$, and for the covering numbers it holds that
\begin{equation}
\begin{aligned}
&\CN_1(\varepsilon/8,{\CF_0},\{x_1,\dots,x_n\}) \leq \CN_1(\varepsilon/16,\widetilde{\CF},\{x_1,\dots,x_n\}),
\end{aligned}
\end{equation}
becuase $\{f_1,\dots,f_k\}$ is an $\varepsilon$-cover of $\CF_0$ if and only if $\{\nicefrac{1}{2}(f_1+1),\dots,\nicefrac{1}{2}(f_k+1)\}$ is an $\nicefrac{\varepsilon}{2}$-cover of $\widetilde{\CF}$.
In the third step we bound the covering numbers with packing numbers by Lemma \ref{lemma:covering-packing}, \citep{gyorfi2002distribution}, as
\begin{equation}
\begin{aligned}
&\CN_1(\varepsilon/16,\widetilde{\CF},\{x_1,\dots,x_n\}) \leq \mathcal{M}_1(\varepsilon/16,\widetilde{\CF},\{x_1,\dots,x_n\}).
\end{aligned}
\end{equation}
Finally, by Theorem \ref{thm:vcbound-expect} we derive an upper bound depending only on pseudo-dimension $V_{\CF_0^{+}}$ and $\varepsilon > 0$, cf. \citep{haussler1995sphere}),
\begin{equation}
\mathcal{M}_1(\varepsilon/16,\widetilde{\CF},\{x_1,\dots,x_n\})  \leq e(V_{\CF_0^{+}}+1)  \bigg(  \frac{32e}{\varepsilon}\bigg)^{\!\!V_{\CF_0^{+}}}\hspace*{-1mm},\hspace*{1mm}
\end{equation}
where we used that $V_{\CF_0^{+}} = V_{\widetilde{\CF}^{+}}$. Combining these inequalities yields that
\begin{equation}
    C\doteq e^2(V_{\CF_0^{+}}+1)^2  \bigg(  \frac{32e}{\varepsilon}\bigg)^{2V_{\CF_0^{+}}}
\end{equation}
is an almost surely valid bound on the covering number of $\CG$ w.r.t.\ the $L_1$ norm.
\end{proof}
    
        Lemma \ref{lemma:covering-bound} and Theorem \ref{thm:ulln} are our main tools to prove Lemma \ref{lemma:uniform-bound}, which is a simplified version of Theorem \ref{theorem:algorithm-I}.
	
	\begin{lemma}\label{lemma:uniform-bound}
		Consider the constructed confidence regions in $\CF_0$, that is let
		\begin{equation}
			\CF_{n}^{(1)}\, \doteq\, \big\{ \, f_\theta \in \CF_0 \;\,|\,\;  \theta \in \Theta_{n}^{(1)}\, \big\}.
		\end{equation}
		Assume A\ref{ass:b1}, A\ref{ass:b2}, A\ref{ass:b3}, A\ref{ass:b4} and A\ref{ass:b6}, then for all\, $\varepsilon >0$
		\begin{equation}\label{eq:model-uniformcons}
			\mathbb{P}\,\bigg(\,\bigcup_{k=1}^{\infty} \bigcap_{n=k}^{\infty} \left\{\, \CF_{n}^{(1)} \subseteq B(f^*, \varepsilon) \,\right\} \bigg) \,=\, 1,
			\vspace{0mm}
		\end{equation}
		where $B(f^*, \varepsilon)\doteq \{ f \in \CF_0 \;|\; \norm{f-f^*}_{\scriptscriptstyle Q}<\varepsilon\}$.
	\end{lemma}

        \begin{proof}
We are going to show that $Z_n^{(0)}(\theta)$ tends to be the greatest in the  ordering, defined by $\prec_\pi$, uniformly for all $\theta \notin B(f^*,\varepsilon)$ as $n\to \infty$. For $j \in [m-1]$ the difference between the reference variables can be bounded from below as follows
\begin{equation}\label{eq:infimum-refs}
\begin{aligned}
&Z_n^{(0)}(\theta) - Z_n^{(j)}(\theta)\, =\;\, \frac{1}{n} \sum_{i=1}^n (f_\theta(X_i) - f_{\theta,n}^{(0)}(X_i))^2 - \frac{1}{n} \sum_{i=1}^n (f_\theta (X_i) - f_{\theta,n}^{(j)}(X_i))^2\\
&=\;\, \frac{1}{n} \sum_{i=1}^n (f_\theta (X_i) - f_{*}(X_i))^2 + \frac{1}{n} \sum_{i=1}^n (f^*(X_i) - f_{*,n}^{(0)}(X_i))^2 \\
&\quad +\frac{2}{n} \sum_{i=1}^n (f_\theta (X_i) - f_{*}(X_i))(f^*(X_i) - f_{*,n}^{(0)}(X_i))-\frac{1}{n} \sum_{i=1}^n (f_\theta (X_i) - f_{\theta,n}^{(j)}(X_i))^2\\
&\geq\;\,\EE \big[(f_\theta (X)-f^*(X))^2\big] - \EE \big[(f_\theta (X)-f^*(X))^2\big] \\
&\quad + \frac{1}{n} \sum_{i=1}^n (f_\theta (X_i) - f_{*}(X_i))^2 +
\EE \big[(f^*(X)-f_{*,n}^{(0)}(X))^2\big] \\
&\quad -\EE \big[(f^*(X)-f_{*,n}^{(0)}(X))^2\big] + \frac{1}{n} \sum_{i=1}^n (f^*(X_i) - f_{*,n}^{(0)}(X_i))^2 \\
&\quad  -\EE \big[(f_{\theta} (X)-f_{\theta,n}^{(j)}(X))^2\big]+ \EE \big[(f_{\theta}(X)-f_{\theta,n}^{(j)}(X))^2\big]\\
&\quad -\frac{1}{n} \sum_{i=1}^n (f_\theta(X_i) - f_{\theta,n}^{(j)}(X_i))^2 -\frac{4}{n} \sum_{i=1}^n \big|\,f^*(X_i)-f_{*,n}^{(0)}(X_i)\,\big| ,
\end{aligned}
\end{equation}
where $X$ is independent of $\CD_0^\pi$ and $\CD_j^\pi(\theta)$ and the expectation is 
w.r.t.\ variables  $X,\CD_0$ and $\CD_j^\pi(\theta)$.
We want to bound the infimum of the left hand side for $f_\theta \in \CF_0 \setminus B(f^*,\varepsilon)$ from below. First, notice that
\begin{align}
\inf_{f_\theta \notin B(f^*,\varepsilon)}\EE \big[(f_\theta (X)-f^*(X))^2\big] \geq \varepsilon^2.
\end{align}
Further, for the expectations of the $\mathcal{L}_2$-errors we have
\begin{equation}\label{eq:37}
\begin{aligned}
&\EE \big[(f^*(X)-f_{*,n}^{(0)}(X))^2\big] \geq 0, \\
&\EE \big[(f_\theta(X)-f_{\theta,n}^{(j)}(X))^2\big] \leq \frac{c_1 + (c_2+c_3 \log(n))V_{\CF_0^{+}}}{n}\doteq a_n,
\end{aligned}
\end{equation}
where we used Theorem \ref{thm:expect} from Appendix \ref{app:theorems} \cite[Theorem 11.5]{gyorfi2002distribution} with the observation that \begin{equation}
\inf_{f \in \CF_0} \int \,|\,f(x) - f^*(x)\,|^2 \dd \QX (x) =0
\end{equation}
follows from $f^* \in \CF_0$. Additionally, Theorem \ref{thm:expect} can be applied directly without the truncation, because $|Y|\leq 1$ and $\sup_{x\in \RR^d}|f(x)|\leq 1$ for all $f \in \CF_0$.

The deviations from the expected values are bounded with a supremum over $\CF_0$ as
\begin{equation}
\begin{aligned}
&\EE \big[ (f_\theta (X)-f^*(X))^2\big] - \frac{1}{n} \sum_{i=1}^n (f_\theta (X_i) - f_{*}(X_i))^2\\
&+ \EE \big[(f^*(X)-f_{*,n}^{(0)}(X))^2\big]  - \frac{1}{n} \sum_{i=1}^n (f^*(X_i) - f_{*,n}^{(0)}(X_i))^2 \\
&- \EE \big[(f_\theta(X)-f_{\theta,n}^{(j)}(X))^2\big]  + \frac{1}{n} \sum_{i=1}^n (f_\theta(X_i) - f_{\theta,n}^{(j)}(X_i))^2 \\
&\leq 3\sup_{f,g \in \CF_0}\Big|\, \frac{1}{n} \sum_{i=1}^n (f(X_i) -g(X_i))^2 - \EE \big[(f(X)-g(X))^2\big] \Big|\notag\\
&=3\sup_{h \in \CG}\Big|\, \frac{1}{n} \sum_{i=1}^n h(X_i) - \EE [h(X)] \Big|,
\label{eq:38}
\end{aligned}
\end{equation}
where $\CG$ is defined in \eqref{eq:h-definition}. The inequality holds because $f_\theta$, $f^*$ and the estimators are all included in $\CF_0$.
For the last term in \eqref{eq:infimum-refs} we apply the Cauchy--Schwartz inequality, then we bound the obtained value with the same supremum as in \eqref{eq:38} plus the zero sequence $a_n$, defined in \eqref{eq:37}, as
\begin{align}
&\Big(\frac{1}{n}\sum_{i=1}^n \big|\,f^*(X_i)-f_{*,n}^{(0)}(X_i)\,\big|\Big)^2 \leq \frac{1}{n^2}\sum_{i=1}^n (f^*(X_i)-f_{*,n}^{(0)}(X_i))^2 \sum_{i=1}^n 1 \notag\\
&\leq \frac{1}{n}\sum_{i=1}^n (f^*(X_i)-f_{*,n}^{(0)}(X_i))^2 - \EE \big[(f^*(X)-f_{*,n}^{(0)}(X))^2\big] + \EE\big[(f^*(X)-f_{*,n}^{(0)}(X))^2\big]   \notag\\
&\leq \sup_{f,g \in \CF_0} \Big|\,\frac{1}{n}\sum_{i=1}^n (f(X_i)-g(X_i))^2 - \EE \big[(f(X)-g(X))^2\big]\,\Big| + a_n\\
&=\sup_{h \in \CG}\Big|\, \frac{1}{n} \sum_{i=1}^n h(X_i) - \EE[h(X)] \Big| + a_n.
\end{align}
For the sake of simplicity let
\begin{equation}\label{eq:55}
A_n\doteq \sup_{h \in \CG}\Big|\, \frac{1}{n} \sum_{i=1}^n h(X_i) - \EE[h(X)]\,\Big|.
\end{equation}
To sum up, we can bound the infimum of \eqref{eq:infimum-refs} from below as
\begin{equation}\label{eq:56}
\begin{gathered}
\inf_{f_\theta \notin B(f^*,\varepsilon)} (Z_n^{(0)}(\theta) - Z_n^{(j)}(\theta)) \geq \varepsilon^2 - a_n -3A_n - 4 \sqrt{ A_n+ a_n}.
\end{gathered}
\end{equation}
Observe that $(a_n)_{n\in\NN}$ is a deterministic sequence and tends to zero as $n\to\infty$. If $(A_n)_{n\in\NN}$ converges to zero almost surely, then for all $j \in [m-1]$ a.s. there exists $n_j \in \NN$ such that for all $n >n_j$ difference $Z_n^{(0)}(\theta)-Z_n^{(j)}(\theta)$  becomes positive for all $f_\theta \notin B(f^*,\varepsilon)$, and by the construction for all $n  > \max_{j\in [m-1]} n_j$ we have $\CF_{\varrho,n}^{(1)} \subseteq B(f^*,\varepsilon)$. 

It remained to prove that $A_n \xrightarrow{\,a.s.\,}0$. We use a uniform law of large numbers. By Lemma \ref{lemma:covering-bound} the assumption of Theorem \ref{thm:ulln} is satisfied, i.e., for all $\varepsilon>0$  we have
\begin{equation}
\begin{gathered}
\sum_{n=1}^\infty \EE\big(\, \CN_1(\varepsilon/8,\CG,\{X_1,\dots,X_n\})\,\big)e^{-\nicefrac{n\varepsilon^2}{(128B^2)}}
\leq C(\varepsilon, V_{\CF_0^{+}})\cdot\sum_{n=1}^\infty e^{-\nicefrac{n\varepsilon^2}{(128B^2)}} < \infty,
\end{gathered}
\end{equation}
where $C(\varepsilon, V_{\CF_0^{+}})$ is defined in the proof of Lemma \ref{lemma:covering-bound}, $B=4$, because $(f(x)-g(x))^2 \in [0,4]$ for all $x\in \BX$ and $f,g \in \CF_0$.
Consequently $A_n \xrightarrow{\,a.s.\,}0$ holds.
\end{proof}

	Lemma \ref{lemma:Lipschitz-cont} formalizes the final step of proving Theorem \ref{theorem:algorithm-I}.
	
	\begin{lemma}\label{lemma:Lipschitz-cont}
		Assume A\ref{ass:b3}, A\ref{ass:b5} and that \eqref{eq:model-uniformcons} holds for all $\varepsilon >0$, then for all $\varepsilon >0$, we have
		\begin{equation}\label{eq:parameter-space-version}
			\mathbb{P}\,\bigg(\,\bigcup_{k=1}^{\infty} \bigcap_{n=k}^{\infty} \left\{\, \Theta_{n}^{(1)} \subseteq B(\theta^*, \varepsilon) \,\right\} \bigg) \,=\, 1.
			\vspace{0mm}
		\end{equation}
	\end{lemma}
	\begin{proof}
If $f_\theta \in B(f^*, \nicefrac{\varepsilon}{L})$, then $\theta \in B(\theta_*, \varepsilon)$, because of A\ref{ass:b5}.
Hence, it follows that
\begin{equation}\label{eq:contain}
\left\{\, \Theta_{\varrho, n}^{(1)} \subseteq B(\theta_*, \varepsilon) \,\right\} \supseteq \left\{\, \CF_{\varrho, n}^{(1)} \subseteq B(f^*, \nicefrac{\varepsilon}{L}) \,\right\}.
\end{equation} 
The inclusion also holds for the $\liminf$ of events.
Finally, since \eqref{eq:model-uniformcons} holds for $\nicefrac{\varepsilon}{L}$ by \eqref{eq:contain}  we also have \eqref{eq:parameter-space-version}.
\end{proof}

\begin{customthm}{\ref{theorem:algorithm-I}}
Assume A\ref{ass:b1}-A\ref{ass:b6}, then for all sample size $n \in \NN$ and integers $q\leq m$ we have $\PP\big(\,\theta^* \in \Theta_{n}^{(1)}\,\big)\, = \, \nicefrac{q}{m}$.
		  In addition, if\, $q <m$, then $(\Theta_{n }^{(1)})_{n \in \NN}$ is strongly uniformly consistent, i.e., by 
        denoting $B(\theta^*, \varepsilon) \doteq \{\hspace{0.3mm}\theta\;|\;\theta \in \Theta,\, \Delta(\theta,\theta^*)< \varepsilon\hspace{0.3mm}\}$, we have for 
        all $\varepsilon > 0$ that
         \begin{equation}\label{eq:strong-uniform-consistency}
         \PP \bigg( \bigcup_{n=1}^\infty \bigcap_{k=n}^\infty \{ \widehat{\Theta}_k^{(1)} \subseteq B(\theta^*, \varepsilon)\}  \bigg) = 1.
         \end{equation}
\end{customthm}

\begin{proof} 
    Observe that the statement regarding the exact coverage probability follows from Theorem \ref{theorem:exact-confidence}, because P\ref{ass:p1} and P\ref{ass:p2} are satisfied by the ERM-based ranking.

    Under A\ref{ass:b6} we have \eqref{eq:bounded-covering}, hence Lemma \ref{lemma:uniform-bound} proves that the confidence regions are uniformly consistent in the model space. By Lemma \ref{lemma:Lipschitz-cont} in this case condition A\ref{ass:b5} is sufficient for uniform consistency in the parameter space.
\end{proof}

\begin{customthm}{\ref{theorem:quantitative-uniform}}
Assume A\ref{ass:b1}-A\ref{ass:b6} and $q <m$. For all $\varepsilon >0$ there exists $N_0$ such that for $n \geq N_0$:
		\begin{equation}\label{eq:PAC-uniform}
			\PP \,\big( \,\Theta_{n}^{(1)} \subseteq B(\theta^*, \varepsilon)\,\big) \,\geq\, 1- C_1\,e^{-n\lambda_1},
		\end{equation}
		where $C_1$ depends only on $\varepsilon$ and  $V_{\CF_0^{+}}$, and $\lambda_1$ depends only on $\varepsilon$ and the Lipschitz constant.
\end{customthm}

\begin{proof} 
Similarly as in the proof of Theorem \ref{theorem:algorithm-I} observe that
\begin{equation}
\begin{gathered}
\big\{ \Theta_{\varrho, n}^{(1)} \subseteq B(\theta_*, \varepsilon) \big\} \supseteq 
 \big\{ \forall j \in [m-1]:\inf_{\theta \notin B(\theta_*,\varepsilon)} \big(Z_n^{(0)}(\theta) - Z_n^{(j)}(\theta)\big) >0 \big\}.
\end{gathered}
\end{equation}
By Lemma \ref{lemma:Lipschitz-cont} and \eqref{eq:56} we obtain for all $j \in [m-1]$ that
\begin{equation}
\begin{aligned}
\inf_{\theta \notin B(\theta_*,\varepsilon)}\big(Z_n^{(0)}(\theta) - Z_n^{(j)}(\theta)\big)&\geq \inf_{f_\theta \notin B(f^*,\varepsilon/L)}\big(Z_n^{(0)}(\theta) - Z_n^{(j)}(\theta)\big)\\
&\geq \frac{\varepsilon^2}{L^2}-a_n - 3 A_n - 4\sqrt{A_n+a_n},
\end{aligned}
\end{equation}
where $L$ is the Lipschitz constant, $(a_n)_{n\in\NN}$ is a deterministic zero sequence defined in \eqref{eq:37} and $A_n$ is the random (measurable) supremum of the deviations defined in \eqref{eq:55}. Notice that the lower bound does not depend on index $j$, consequently by elementary manipulations
\begin{equation}
\hspace{-2.6mm}
\begin{aligned}
&\;\PP \big(\, \Theta_{\varrho,n}^{(1)} \subseteq B(\theta_*, \varepsilon)\,\big)\geq \PP \big( \forall j \in [m-1]:\inf_{\theta \notin B(\theta_*,\varepsilon)} \big(Z_n^{(0)}(\theta) - Z_n^{(j)}(\theta)\big) >0 \big)\hspace*{-5mm}\\
&\geq\PP \Big(\, \forall j \in [m-1]:\inf_{f_\theta \notin B(f^*,\varepsilon/L)} \big(Z_n^{(0)}(\theta) - Z_n^{(j)}(\theta)\big) >0 \,\Big)\\
&\geq \PP \bigg(\, \frac{\varepsilon^2}{2L^2}> a_n + 3A_n + 4 \sqrt{A_n + a_n} \,\bigg)\\
&\geq \PP \bigg(\, \frac{\varepsilon^2}{6L^2}> a_n,\frac{\varepsilon^2}{6L^2}\geq  3A_n, \frac{\varepsilon^2}{6L^2}\geq 4 \sqrt{A_n + a_n} \,\bigg) \\
&= \PP \bigg(\, \frac{\varepsilon^2}{6L^2}> a_n,\frac{\varepsilon^2}{18L^2}\geq A_n, \frac{\varepsilon^4}{24^2L^4}\geq A_n + a_n \,\bigg)\\
&\geq \PP \bigg( \min\Big(\frac{\varepsilon^2}{6L^2},\frac{\varepsilon^4}{2 \cdot 24^2 L^4}\Big)> a_n,\min\Big(\frac{\varepsilon^2}{18L^2},\frac{\varepsilon^4}{2 \cdot 24^2L^4}\Big) \geq  A_n \bigg).
\end{aligned}
\end{equation}
Since $(a_n)$ is deterministic and tends to zero there exists $n_0$ such that for all $n >n_0$: 
\begin{equation}\label{eq:27}
\min\Big(\frac{\varepsilon^2}{6L^2},\frac{\varepsilon^4}{2 \cdot 24^2 L^4}\Big)> a_n
\end{equation}
holds. Furthermore let
\begin{equation}
\tau\doteq\min\Big(\frac{\varepsilon^2}{18L^2},\frac{\varepsilon^4}{2 \cdot 24^2L^4}\Big),
\end{equation}
then by Theorem \ref{thm:ulln} and Lemma \ref{lemma:covering-bound} for $n > n_0$ we have
\begin{equation}
\begin{aligned}
&\PP( A_n \leq \tau) = 1- \PP( A_n > \tau)\geq 1 - 8 \cdot C \cdot \exp\Big(-\frac{n\tau^2}{128B^2}\Big),
\end{aligned}
\end{equation}
where $B=4$. Consequently with constants $C_1 \doteq 8C$ and  $\lambda_1\doteq \frac{\tau^2}{128B^2}$ Theorem  \ref{theorem:quantitative-uniform} follows.
\end{proof}
		We can quantify the data size which is required for the PAC bound to hold. 
		The only condition, which needs to be guaranteed in the proof, is \eqref{eq:27}. We can show that 
		\begin{equation}
			\begin{gathered}
				N_0 \doteq \bigg(\frac{c_1 + (c_2 + c_3) V_{\CF_0^{+}} }{\varepsilon'}\bigg)^{\!2} <\,n, \;\;\text{where} \;\; \varepsilon' \,\doteq\, \min\bigg(\frac{\varepsilon^2}{6L^2},\frac{\varepsilon^4}{2 \cdot 24^2 L^4}\bigg)
			\end{gathered}
		\end{equation}
		and constants $c_1$, $c_2$ and $c_3$ are defined in \cite[Theorem 11.5]{gyorfi2002distribution},
		is sufficient to satisfy \eqref{eq:27}.

\section{Model Classes with Finite Pseudo-Dimension}\label{app:model-classes}

The ERM-based method can be applied on a wide variety of model classes. In the main paper we considered the perceptron, as a demonstrative example. In this section several model classes are presented in which the ERM-based method constructs uniformly consistent abstract confidence regions for the true regression function.

\subsection{Linear Models}
	Let $\CF_0$ be a subset of a $k$ dimensional vector space, $\widetilde{\CF}$, spanned by an orthonormal basis $\varphi_1,\dots,\varphi_k$ of $\BX \to [-1,1]$ type functions, i.e.,
	\begin{equation}
		f_\theta(x)\, =\, \sum_{i=1}^k \theta_i\, \varphi_i(x),
	\end{equation}
	for some $\theta \in \RR^k$, thus $\Theta \subseteq \RR^k$. We know that the range of the  regression function is limited to $[-1,1]$, therefore we consider the supremum ball with radius $1$ in linear space $\widetilde{\CF}$. It is known that $V_{\widetilde{\CF}^{+}} \leq k+1$, cf. \cite[Theorem 9.5]{gyorfi2002distribution}, hence A\ref{ass:b6} holds. The inverse Lipschitz condition is verified by Lemma \ref{lemma:inverse-Lipschitz}, therefore in this framework our method builds exact, strongly uniformly consistent confidence regions for the true regression function.
	
	We should note that there is no explicit formula for the optimal solution of the empirical risk minimization problem over $\CF_0$ when it is not the entire linear space $\widetilde{\CF}$, in which case the ordinary least squares (OLS) estimator provides analitical solution under mild conditions. The range of the OLS estimator might include elements that are not in the desired $[-1,1]$ interval. This problem can be resolved in several ways. Obviously, we can truncate the OLS estimator into the $[-1,1]$ interval. The advantage of this approach is that generalizing the theory to these variants of our method is fairly easy, because truncation does not increase the VC dimension of a model class. On the other hand one can argue that our estimator should be included in the given model class. Therefore, another useful approach is to restrict the parameters to ensures the boundedness of the linear combination in $[-1,1]$, for example, we can require the estimator to be a convex combination of the basis functions. 

\begin{lemma}\label{lemma:inverse-Lipschitz}
Let $\CF_0$ be a finite dimensional linear subspace of $\mathcal{L}_2(\QX)$ spanned by an orthonormal system of vectors $\{\varphi_1, \dots,\varphi_k\}$. Then the linear parameterization is an isometry, i.e., for all $\alpha,\beta \in \RR^k$ we have
\vspace{-1mm}
\begin{equation}\label{eq:lin_par}
\norm{\alpha-\beta}_2 =\Big\| \sum_{i=1}^k \alpha_i \varphi_i(x) - \sum_{i=1}^k \beta_i \varphi_i(x) \Big\|_{\scriptscriptstyle Q},
\end{equation}
henceforth it is also inverse Lipschitz-continuous.
\end{lemma}
\begin{proof}
The orthonormality of $\varphi_1, \dots, \varphi_k$ implies \eqref{eq:lin_par}. 
\end{proof}

\subsection{Radial Basis Function Networks}
	In the second example we examine the model class of radial basis function networks of \cite{lowe1988multivariable} with one hidden layer on $\BX \subseteq \RR^d$.
	
	Let $K: \RR^{+} \to \RR$ be a local averaging kernel function. 
    The most popular choices are the na{\"i}ve kernel $K(z) \doteq \BI(\,z \in [\hspace{0.3mm}0,1)\,)$ and the Gaussian kernel $K(z)\doteq  e^{-z^2}$. We consider regular kernels, which are nonnegative, monotonically decreasing, left continuous and satisfy
	\begin{equation}
		\begin{gathered}
			\int_{\RR^d} K(\norm{x}) dx \neq 0\qquad\text{and}\qquad
			\int_{\RR^d} K(\norm{x}) dx < \infty,
		\end{gathered}
	\end{equation}
	with the Euclidean norm $\norm{\cdot}$. 
    It is easy to prove that regular kernels are bounded \citep{krzyzak1997radial, lei2014generalization}.  Let $K(z) \leq K_*$ for all $z \in \RR^{+}$ and $k$ be the number of computational units, i.e., consider the model class
	\begin{equation}
		\begin{aligned}
			\CF_0\, &\doteq\, \Big\{ \sum_{i=1}^k w_i K( \norm{x-c_i}_{A_i}) + w_0 \;\,\big|\;\,\\
            & \quad w_1, \dots,w_n \in \RR,
			c_1,\dots, c_k \in \RR^d, A_1,\dots, A_k \in \RR^{d \times d}, \sum_{i=1}^k |w_i| < B\Big\},
		\end{aligned}
	\end{equation}
	where $\norm{x-c_i}_{A_i} \doteq (x-c_i)\tr A (x-c_i)$.
	Restriction $\sum_{i=1}^k |w_i| < B$ is made both to control the boundedness of the models and the complexity or capacity of the model class.
	
	The verification of the inverse Lipschitz property is challenging and depends on the specific choice of the kernel, but the uniform convergence in $\mathcal{L}_2(\QX)$ can be guaranteed by bounding the proper complexity measure. We refer to the result in \cite[Lemma 17.2]{gyorfi2002distribution} where the random $\mathcal{L}_1$-cover of $\CF_0$ is bounded independently of the sample as
	\begin{equation}
		\begin{aligned}
			\CN_1 ( \nicefrac{\varepsilon}{8},\CF_0,\{x_1,\dots,x_n\}) \leq\, 3^k \Big( \frac{96 e B (k+1)}{\varepsilon}\Big) ^{(2d^2+2d +3)k+1} \doteq\, R.
		\end{aligned}
	\end{equation}
	Combining this inequality with
	\begin{equation}
		\!\!\!\!\CN_1(\varepsilon,\CH,\{x_1,\dots,x_n\}) \leq \big(\CN_1(\varepsilon/8,\CF_0,\{x_1,\dots,x_n\})\big)^2\!\!,\!\!
	\end{equation}
	which is proved in Appendix \ref{app:algorithm-I},
	yields that 
	\begin{equation}
		\CN_1(\varepsilon, \CH, \{X_1,\dots,X_n\})\, \leq\, R^2.
	\end{equation}
	This is sufficient for the uniform consistency  in  $\mathcal{L}_2(\QX)$.
	
	The training of radial basis function networks suffers from high computational burden and effective methods for finding the global optimum are only known in special cases. Nevertheless, in practice stochastic gradient type approaches are widely used and provide estimates with good performance. These estimates can also be used as $\{f_{\theta,n}^{(j)}\}$ functions to define a ranking, and the exact confidence of the constructed region is guaranteed.
	
	\subsection{Neural Networks}
	
	 Deep neural networks are a generalizations of the single neuron model and are widely used in many different areas including industry, economics and health care. We consider a class of feed-forward models, however, we emphasize that our method works for any fixed class of neural networks with finite pseudo-dimension.
	Let $\CF_0$ be a class of feed-forward neural networks with a fixed layer structure containing $W$ parameters (weights and biased terms together), $M$ layers and $k$ computation units. Either the sigmoid function
	\begin{equation}
		\sigma_S(x)\, \doteq\, \frac{1}{1+ e^{-x}},
	\end{equation}
	or the recently favored rectified linear unit (ReLU) 
	\begin{equation}
		\sigma_R(x)\, \doteq\, x \cdot \BI(x\geq 0),
	\end{equation}
	can be used as activation in each computation unit.
	
	By \cite[Theorem 14.2]{anthony2009neural} for neural networks with sigmoid activation functions we have
	\begin{align} \label{eq:35} 
		\hspace{-0.15cm}V_{\CF_0^{+}} \hspace*{-1mm} \leq ((W+2)k)^2  \hspace*{-1mm}+ \hspace*{-1mm} 11(W+2)k \log_2(18(W+2)k^2).
	\end{align}
	Similary for the ReLU activation it is proved by \cite[Theorem 7]{bartlett2019nearly} that under general conditions the pseudo-dimension is $O (\,W M \log(W)\,)$, therefore
	uniform convergence in $\mathcal{L}_2(\QX)$ is a corrolary of Lemma \ref{lemma:uniform-bound}.
	For a particular example when $\BX \subseteq \RR^d$ we consider a network with one hidden layer, i.e., let the model class be defined by
	\begin{equation}
		\begin{gathered}
			\CF_0\, \doteq\, \Big\{ \, \sum_{i=1}^k c_i \sigma_S(a_i^\intercal x + b_i) + c_0 \,\big|\, \sum_{i=1}^k |c_i| \leq B,
			\, c_0,\dots c_k, b_1, \dots, b_k \in \RR, a_1, \dots, a_k \in \RR^d \,\Big\},
		\end{gathered}
	\end{equation}
	where condition $\sum_{i=1}^k |c_i| \leq B$ ensures the boundedness of the models. In this particular case the number of parameters is $W = (k+1)+ k + k\cdot d$, and because of \eqref{eq:35}, the pseudo-dimension of $\CF_0$ is finite.
	
	The training of artificial neural networks faces several computational and theoretical challenges, because in general it requires solving highly non-convex optimization problems. Approximate solutions (e.g., which are close to a local optimum) can be found by stochastic gradient descent type algorithms (cf.\ backpropagation). These solutions provide feasible estimates with decent performances, which can also be used in our uncertainty quantification framework as $\{f_{\theta,n}^{(j)}\}$ functions to define the ranking, leading to exact coverage guarantees.

\section{Proof of Theorem \ref{theorem:algorithm-II} and Theorem \ref{theorem:quantitative-knn}}\label{app:algorithm-II}

\begin{customthm}{\ref{theorem:algorithm-II}}
            Assume A\ref{ass:b1}-A\ref{ass:b4}, then for all sample size $n \in \NN$ and integers $q\leq m$ we have $\PP\big(\,\theta^* \in \Theta_{\varrho,n }^{(2)}\,\big) = \nicefrac{q}{m}$.
	    Further, if $q <m$, 
		$\nicefrac{k_n}{n} \to 0$, and there exists $\delta >0$ such that $\nicefrac{k_n^2}{n^{(1+\delta)}} \to \infty$ as $n\to\infty$, then $(\,\Theta_{n }^{(2)}\,)_{n \in \NN}$ is strongly pointwise consistent, i.e., for $\theta\neq\theta^*$:
        \begin{equation}
            \PP\bigg( \bigcap_{n=1}^\infty \bigcup_{k=n}^\infty \big\{ \theta \in \widehat{\Theta}_k^{(2)}\big\} \bigg) = 0
        \end{equation}
            holds for all distribution of $(X,Y)$ such that
		$\PP\big(\hspace{0.3mm}\norm{X- X'}_2 \neq \norm{X-X''}_2\hspace{0.3mm}\big) =1,$
		where $X'$ and $X''$ are independent copies of $X$.
\end{customthm}

\begin{proof} 
We are going to show that for $\theta \neq \theta_*$ reference variable $Z_n^{(0)}(\theta)$ tends to be the greatest in the ordering.
The proof will be presented in two steps. First, we reduce the problem to the almost sure convergence of empirical $\mathcal{L}_1$-errors to zero, then we prove that the convergence holds for the kNN estimates.

Let $\theta \neq \theta_*$ and fix an index $j \in [m-1]$. We bound the difference between the reference variables from below as
\begin{equation}\label{quantity}
\begin{aligned}
&Z_n^{(0)}(\theta)- Z_n^{(j)}(\theta) = \frac{1}{n}\sum_{i=1}^n\big(f_\theta(X_i)- f_{*,n}^{(0)}(X_i)\big)^2 - \frac{1}{n}\sum_{i=1}^n\big(f_\theta(X_i)- f_{\theta,n}^{(j)}(X_i)\big)^2 \\
&= \frac{1}{n}\sum_{i=1}^n\big(f_\theta(X_i)- f^*(X_i)\big)^2 +\frac{1}{n}\sum_{i=1}^n \big(f^*(X_i)- f_{*,n}^{(0)}(X_i)\big)^2 \\
&\quad +\frac{2}{n}\sum_{i=1}^n \big( f_\theta(X_i) -f^*(X_i)\big) \big( f^*(X_i) - f_{*,n}^{(0)}(X_i)\big)\\
&\quad- \frac{1}{n}\sum_{i=1}^n\big(f_\theta(X_i)- f_{\theta,n}^{(j)}(X_i)\big)^2\\
&\geq \frac{1}{n}\sum_{i=1}^n\big(\,f_\theta(X_i)- f^*(X_i)\,\big)^2-\frac{2}{n}\sum_{i=1}^n \big|\,f^*(X_i)- f_{*,n}^{(0)}(X_i)\,\big| 
\\
&\quad -\frac{4}{n}\sum_{i=1}^n \big|\, f^*(X_i)- f_{*,n}^{(0)}(X_i)\,\big| 
- \frac{2}{n}\sum_{i=1}^n\big|\,f_\theta(X_i)- f_{\theta,n}^{(j)}(X_i)\,\big|,
\end{aligned}
\end{equation}
where we used that $f_\theta$, $f^*$ and the estimators are all bounded in $[-1,1]$.  Observe that by the strong law of large numbers as $n \to \infty$ we have
\begin{equation}
\frac{1}{n}\sum_{i=1}^n\big(\,f_\theta(X_i)- f^*(X_i)\,\big)^2 \xrightarrow{\text{a.s.}} \kappa^2 \doteq \norm{f_\theta- f^*}_{\scriptscriptstyle Q}^2,
\end{equation}
where $\kappa >0$, because of A\ref{ass:b2}.
It remained to prove that both
\begin{align}\label{eq:convergences1}
&\frac{1}{n}\sum_{i=1}^n \big|\,f^*(X_i)- f_{*,n}^{(0)}(X_i)\,\big|\xrightarrow{\,a.s.\,}0, \quad
\frac{1}{n}\sum_{i=1}^n\big|\,f_\theta(X_i)- f_{\theta,n}^{(j)}(X_i)\,\big| 
\xrightarrow{\,a.s.\,}0
\end{align}
hold. As in the proof of Theorem \ref{theorem:algorithm-I}, then  a.s. there exists  $n_j$ such that the quantity in \eqref{quantity} becomes positive, therefore for all $n>\max_{j\in [m-1]} n_j$ we have $\theta \notin \Theta_{\varrho,n}^{(2)}$.

Notice that it is sufficient for \eqref{eq:convergences1} to prove for any possible regression function $f$ that
\begin{equation}\label{eq:L-1error}
\frac{1}{n}\sum_{i=1}^n \big|\,f_{n}(X_i) - f (X_i)\,\big|\xrightarrow{\,a.s.\,}0
\end{equation}
holds, where $f_{n}$ is the kNN estimator of $f$ based on $\{(X_i, Y_i)\}_{i=1}^n$ with $k_n$ neighbors. Hence, we need to prove the consistency of an empirical $\mathcal{L}_1$-error for all distributions of $(X,Y)$ such that $\PP\big(\hspace{0.3mm}\norm{X- X'}_2 \neq \norm{X-X''}_2\hspace{0.3mm}\big) =1$ holds.
In order to obtain \eqref{eq:L-1error} we use the decomposition 
\begin{align}
&\frac{1}{n}\sum_{i=1}^n \big|\,f_{n}(X_i) - f (X_i)\,\big|\\
&= \frac{1}{n}\sum_{i=1}^n \big|\,f_{n}(X_i) - f (X_i)\,\big|- \EE\Big[\frac{1}{n}\sum_{i=1}^n \big|\,f_{n}(X_i) \hspace*{-0.5mm}-\hspace*{-0.5mm} f (X_i)\,\big|\Big] + \EE\Big[\frac{1}{n}\sum_{i=1}^n \big|\,f_{n}(X_i) \hspace*{-0.5mm}- \hspace*{-0.5mm}f (X_i)\,\big|\Big]\notag\\
&\leq \Big|\frac{1}{n}\sum_{i=1}^n \big|\,f_{n}(X_i) - f (X_i)\,\big| - \EE\Big[\frac{1}{n}\sum_{i=1}^n \big|\,f_{n}(X_i) - f (X_i)\,\big|\Big]\Big|+\EE\Big[\frac{1}{n}\sum_{i=1}^n \big|\,f_{n}(X_i) - f (X_i)\,\big|\Big].\notag
\end{align}
First, we prove that the expected value converges to zero, then we apply McDiarmid's inequality, \citep{mcdiarmid1989method}, to show that the empirical $\mathcal{L}_1$-error is concentrated around its expectation. By the linearity of the expectation and the i.i.d.\ property of the sample
\begin{equation}\label{eq:82}
\begin{aligned}
&\EE \Big[\frac{1}{n}\sum_{i=1}^n \big|\,f_{n}(X_i) - f (X_i)\,\big|\Big] = 
\EE \Big[\big|\,f_{n}(X_n) - f (X_n)\,\big|\Big]\\
&=\EE \Big[\,\Big|\,\frac{1}{k_n} \sum_{i=1}^n Y_i \BI(X_i \in N_n(X_n,k_n))- f(X_n) \,\Big|\,\Big]\\
&= \EE  \Big[\,\Big| \frac{1}{k_n}\sum_{i=1}^n(Y_i -f(X_n))\BI(X_i \in N_n(X_n,k_n))\,\Big|\,\Big]\\
&\leq \EE \Big[\,\Big|\frac{1}{k_{n}}\sum_{i=1}^{n-1} (Y_i - f(X_n)) \BI(X_i \in N_n(X_n,k_n))\Big|+ \frac{1}{k_n}\EE\big[\big|Y_n- f(X_n)\big|\big]\,\Big]\\
&\leq \EE \Big[\,\Big|\frac{1}{k_{n}-1}\!\sum_{i=1}^{n-1} (Y_i - f(X_n)) \BI(X_i \in N_{n-1}(X_n,k_n\!-1))\Big| \,\Big]+ \! \frac{2}{k_n}\\
&\leq \EE  \Big[\big|f_{n-1}'(X_n)- f(X_n)\big|\Big] + \frac{2}{k_n},
\end{aligned}
\end{equation}
where $f_{n-1}'$ is a kNN estimate based on $\{(X_i,Y_i)\}_{i=1}^{n-1}$ with $k_n-1$ neighbors. Clearly, $\nicefrac{2}{k_n} \to 0$ as $n\to \infty$. In addition, by Theorem \ref{thm:weak-consistency}, we have 
\begin{equation}
\begin{aligned}
&\EE\big[\,\big|f_{n-1}'(X_n)- f(X_n)\big|\,\big] \leq \sqrt{\EE\big[\,\big|f_{n-1}'(X_n)- f(X_n)\big|^2\,\big]}  \to 0,
\end{aligned}
\end{equation}
because $(k_n-1) \to \infty$ and $\nicefrac{(k_n-1)}{n}\to 0$ as $n\to \infty$ and  $X_n$ is independent of $\{(X_i,Y_i)\}_{i=1}^{n-1}$. 

In the last step we apply McDiarmid's inequality, Theorem \ref{mcdiarmid}, to bound the difference between the empirical error and the expected value. Let $\CD\doteq \{(x_i,y_i)\}_{i=1}^n$ and $g: (\BX \times \BY)^n \to \RR^+$ be the following function
\begin{equation}
\begin{gathered}
g((x_1,y_1),\dots,(x_n,y_n))\doteq \frac{1}{n}\sum_{i=1}^n \big|\,f_{n}^\CD(x_i) - f (x_i)\,\big|,
\end{gathered}
\end{equation}
where $f_n^\CD(x) \doteq \nicefrac{1}{k_n}\sum_{i=1}^n y_i \BI(x_i \in N(x,k_n))$. Observe that $g$ satisfies the {\em bounded difference} property \eqref{eq:bounded-difference}. If two samples, $\CD_1$ and $\CD_2$, differ only in one pair $(x_i,y_i)$, then 
\begin{equation}
\big|f_n^{\CD_1}(x)- f_n^{\CD_2}(x)\big| \leq \nicefrac{2}{k_n}
\end{equation}
holds for all $x \in \BX$. Hence, by the reverse triangle inequality
\begin{equation}
|\,g(\CD_1) - g(\CD_2)\,| \leq \frac{2}{k_n}.
\end{equation}
Consequently, by Theorem \ref{mcdiarmid} we obtain
\begin{equation}
\begin{aligned}
\PP\big(\, \big|g(\CD)- \EE [g(\CD )]\big| \geq \varepsilon \,\big) 
\leq 2 \exp\bigg( -\frac{k_n^2 \varepsilon^2}{2n}\bigg).
\end{aligned}
\end{equation}
By $\nicefrac{k_n^2}{n^{1+\delta}} \to \infty$ for $n$ large enough
\begin{equation}
\exp\bigg( -\frac{k_n^2 n^\delta\varepsilon^2}{2n^{1+\delta}}\bigg)\leq  \exp\bigg( -\frac{n^\delta\varepsilon^2}{2}\bigg),
\end{equation}
where the right hand side is summable. 
Then, the Borel--Cantelli lemma yields that
\begin{equation}\label{eq:70}
\limsup_{n\to \infty} \big| g(\CD) - \EE [g(\CD)]\big| < \varepsilon
\end{equation}
almost surely and consequently
\begin{equation}
\big|g(\CD) - \EE[g(\CD)]\big|\xrightarrow{\,a.s.\,}0, 
\end{equation}
hence the convergence in \eqref{eq:L-1error} is proved.
\end{proof}

\begin{customthm}{\ref{theorem:quantitative-knn}}
Assume that A\ref{ass:b1}-A\ref{ass:b4} and $q <m$ hold. In addition, $\nicefrac{k_n}{n} \to 0$, and there exists $\delta >0$ such that $\nicefrac{k_n^2}{n^{(1+\delta)}} \to \infty$ as $n\to\infty$. For every distribution of $(X,Y)$ such that\,
		$\PP\big(\norm{X- X'}_2 \neq \norm{X-X''}_2\big) =1,$
		where $X'$ and $X''$ are independent copies of $X$, for a given $\theta \in \Theta$, $\theta\neq \theta^*$ there exists $N_0$ such that for $n \geq N_0$  we have
		\begin{equation}\label{eq:PAC-pointwise}
			\PP\big(\, \theta \notin \Theta_{n}^{(2)}\,\big)\,\geq\, 1 -C_2\, e^{-\nicefrac{k_n^2\lambda_2}{n}},
		\end{equation}
		where $C_2$ depends only on $m$, the number of datasets, and $\lambda_2$ depends on $\kappa \doteq \|f_\theta - f^*\|_{\scriptscriptstyle Q}$.
\end{customthm}
\begin{proof}
For the sake of simplicity let
\begin{equation}
\begin{aligned}
&I_n\doteq \frac{1}{n} \sum_{i=1}^n (f_\theta(X_i) - f^*(X_i))^2, \\
&J_n\doteq \frac{6}{n} \sum_{i=1}^n \big|\,f^*(X_i) - f_{*,n}(X_i)\,\big|,\\
&K_n^{(j)}\doteq \frac{2}{n} \sum_{i=1}^n \big|\,f_\theta(X_i) - f_{\theta,n}^{(j)}(X_i)\,\big|
\end{aligned}
\end{equation}
for $j \in [m-1]$. 
Since 
$Z_n^{(0)}(\theta)- Z_n^{(j)}(\theta)$ is lower bounded by $I_n - J_n - K_n^{(j)}$, cf.
\eqref{quantity}, we have
\begin{equation*}
\begin{gathered}
\{ \theta \notin \Theta_{\varrho,n}^{(2)} \} \supseteq 
\big\{ \forall j \in [m-1]:I_n - J_n - K_n^{(j)} > 0\big\}.
\end{gathered}
\end{equation*}
Let $\kappa \doteq \|f_\theta - f^*\|_{\scriptscriptstyle Q}^2$. Then, 
by De Morgan's law and the union bound, we have
\begin{equation*}
\begin{aligned}
&\PP(\theta \notin\Theta_{\varrho,n}^{(2)})\geq
\PP( \forall j \in [m-1]:I_n - J_n - K_n^{(j)} >0)\\[1.4mm]
&\geq \PP(I_n \geq \nicefrac{\kappa}{2}, J_n < \nicefrac{\kappa}{4},\forall j \in [m-1]: K_n^{(j)} < \nicefrac{\kappa}{4})\\
&\geq 1 - \PP(I_n <\nicefrac{\kappa}{2}) - \PP(J_n \geq \nicefrac{\kappa}{4}) - \sum_{j=1}^{m-1} \PP(K_n^{(j)} \geq \nicefrac{\kappa}{4}).
\end{aligned}
\end{equation*}
We will bound these terms separately. First, let 
\begin{equation}
\xi_i \doteq (f_\theta(X_i) -f^*(X_i))^2
\end{equation}
for all $i \in [n]$. By Hoeffding's inequality
\begin{equation}
\begin{aligned}
&\PP (\, I_n < \nicefrac{\kappa}{2}\,) = \PP ( \,I_n - \kappa < -\nicefrac{\kappa}{2}\,) \\
&\leq  \PP \Big( \,\Big|\frac{1}{n} \sum_{i=1}^n \xi_i - \EE [\xi_1]\Big| > \nicefrac{\kappa}{2}\,\Big) \leq 2e^{-\nicefrac{2n \kappa^2 }{4^2}}\leq 2e^{-\nicefrac{k_n^2 \kappa^2 }{8n}}.
\end{aligned}
\end{equation}
For the second term notice that
\begin{equation}\label{eq:50}
\begin{gathered}
\{ J_n \geq \nicefrac{\kappa}{4}\} \subseteq
\big\{ | J_n - \EE[J_n]| \geq |\,\nicefrac{\kappa}{4} - \EE[ J_n]\,| \big\}\Big).
\end{gathered}
\end{equation}
For $n$ large enough, by \eqref{eq:82} and Theorem \ref{thm:weak-consistency}, $ \EE [J_n] \leq \nicefrac{\kappa}{8}$ holds. Furthermore, by McDiarmid's inequality, see the proof of Theorem \ref{theorem:algorithm-II}, we have
\begin{equation}\label{eq:45}
\begin{gathered}
\PP\,\big(\,|J_n - \EE [J_n]| \geq \nicefrac{\kappa}{8}\,\big) \leq\exp \Big( -\frac{k_n^2 \kappa^2}{2\cdot 48^2n}\Big).
\end{gathered}
\end{equation}
Combining these yields that for $n$ large enough
\begin{equation}\label{eq:46}
\begin{gathered}
\PP \,( \,J_n \geq \nicefrac{\kappa}{4}\,) \leq  \PP \,\big( \,| J_n - \EE J_n| \geq \nicefrac{\kappa}{8}\,\big) 
\leq 2\exp \Big( -\frac{k_n^2 \kappa^2}{2\cdot 48^2n}\Big)
\end{gathered}
\end{equation}
For the last term observe that $\{K_n^{(j)}\}_{j=1}^{m-1}$ are identically distributed, therefore 
\begin{equation}
\sum_{j=1}^{m-1} \PP\,(\,K_n^{(j)} \geq \nicefrac{\kappa}{4}\,) = (m-1)\,\PP\,(\,K_n^{(1)} \geq \nicefrac{\kappa}{4}\,).
\end{equation}
In addition, since $K_n^{(j)}$ is a similar quantity as $J_n$ for a different model parameter, the same argument as above yields the exponential upper bound on the probability of $\big\{K_n^{(j)} \geq \nicefrac{\kappa}{4}\big\}$ for $n$ large enough.
We obtain \eqref{eq:PAC-pointwise}, by merging these exponential bounds together as
\begin{equation}
\PP \big( \, \theta \notin \Theta_{\varrho,n}^{(2)} \,\big) \geq 1 - (4 + 2(m-1) ) \exp\Big(-\frac{k_n^2\lambda_2}{n}\Big),
\end{equation}
for $n$ large enough, where $\lambda_2 \doteq \frac{\kappa^2}{2\cdot 48^2}$.
\end{proof}

\section{Confidence Ellipsoids for Logistic Regression}\label{app:mle}
        In the numerical experiments we compare the novel resampling methods to asymptotic confidence ellipsoids around the maximum likelihood estimator (MLE).    
	We consider a parametric class of models used by logistic regression, which is a standard nonlinear technique to estimate the regression function of binary classification. In this approach the binary output values are usually denoted by $0$ and $1$. As above, let the i.i.d.\ sample be denoted by $\CD_0= \{(X_i,Y_i)\}_{i=1}^n$ and $X_i \in \RR^d$ for each $i =1, \dots, n$. Let $\theta \doteq (\beta,a)$ for some $\beta \in \RR^d$ and $a \in \RR$. Logistic regression models the conditional probabilities of a chosen class with a logistic (sigmoid) function by maximizing the (quasi) conditional likelihood, or equivalently the joint log-likelihood function $L(\theta) = L(\beta,a)\doteq \log\Big(\prod_{i=1}^n p_i(\beta,a)^{Y_i}(1-p_i(\beta,a)^{Y_i}\Big)$,
	where
	\begin{equation}
		p_i(\beta,a)\doteq \PP(Y_i= 1 \,|\,X_i) = \frac{1}{1 + e^{-(X_i\tr\beta +a)}},
	\end{equation}
	for $\beta \in \RR^d$ and $a \in \RR$. Let the (quasi) maximum likelihood estimator (MLE) be denoted by
	\begin{equation}
		\hat{\theta}_n\doteq (\hat{\beta},\hat{a}) \in \argmax L(\beta,a).
	\end{equation}
	By the central limit theorem it is proved that under some regularity conditions, the limit distribution of the MLE is normal, \citep{lehmann2006theory}, i.e.,
	\begin{equation}
		\sqrt{n}\,(\,\hat{\theta}_n - \theta^*\,)\,\xlongrightarrow{\,d\,}\, \CN_{d+1}(0,I(\theta^*)^{-1}),
	\end{equation}
	where $I(\theta^*)$ denotes the Fisher-information matrix. Because of the continuity of the Fisher-information an asymptotic confidence region, in fact an ellipsoid, for $\theta^*$ with significance level $\delta$ can be constructed by 
	\begin{equation}
		\widehat{\Theta}_n \doteq \big\{\, \theta \in \RR^{d+1} \;\,|\,\; (\theta- \hat{\theta}_n) I(\hat{\theta}_n) (\theta- \hat{\theta}_n) \leq \nicefrac{c}{n}\, \big\}, 
	\end{equation}
	where $c$ is the $(1-\delta)$-quantile of the $\chi^2(d+1)$ distribution. 
	Under some mild regularity conditions, the Fisher-information matrix can be computed as
	\begin{equation}
		I(\theta) = - \EE_\theta \big[\partial^2 l_i(\theta)\big],
	\end{equation}
	where $l_i(\theta)\doteq \log(p_i(\theta)^{Y_i}(1-p_i(\theta)^{Y_i})$ denotes the log-likelihood function of a single sample point for $i \in [n]$ and $\partial^2$ is the second derivative or Hessian matrix with respect to $\theta$. By
	\begin{equation}
		L(\theta) = \sum_{i=1}^n l_i(\theta),
	\end{equation}
	the Fisher-information can be approximated from
	\begin{equation}
		I(\theta^*) \approx\frac{\partial^2 L(\hat{\theta}_n)}{n}
	\end{equation}
	and an approximate confidence ellipsoid can be constructed by
	\begin{equation}
		\widehat{\Theta}_n \doteq \big\{ \theta \in \RR^{d+1} \,|\, (\theta- \hat{\theta}_n)\, \partial^2 L(\hat{\theta}_n) \,(\theta- \hat{\theta}_n) \leq {c} \big\}. 
	\end{equation}

\section{Fundamental Definitions and Results Used in the Proofs}\label{app:theorems}

The celebrated Vapnik--Chervonenkis dimension, which is used to define the notion of pseudo-dimension, is one of the most widely applied concepts in statistical learning theory. Let $\CA$ be a class of subsets of $\RR^d$, then $\CA$ shatters a set $C \subseteq \RR^d$ if for all possible subsets $C_0 \subseteq C$, there exists $A \in \CA$ such that $A \cap C = C_0$.
	\begin{definition}[VC dimension]
		The VC dimension of $\CA$, denoted by $V_\CA$, is the largest number $n \in \NN$ such that there exists a set $\{z_1,\dots,z_n\}$ of vectors in $\RR^d$ which is shattered by $\CA$; $V_\CA \doteq \infty$ if the maximum cardinality does not exist.
	\end{definition}
    
The lemma that follows establishes a close relationship between packing and covering numbers \cite[Lemma 9.2]{gyorfi2002distribution}. It was used in the proof of Lemma \ref{lemma:covering-bound}.
\begin{lemma}\label{lemma:covering-packing}
Let $\CF$ be a class of functions on $\RR^d$, $\QX$ be a probability measure on $\RR^d$ and $\norm{\cdot} \doteq \norm{\cdot}_{\mathcal{L}_1(\QX)}$ be the $\mathcal{L}_1$-norm on $\CF$.  Then for all $\varepsilon >0$:
\begin{equation}
\CN (\varepsilon, \CF,\norm{\cdot}) \leq \mathcal{M}(\varepsilon, \CF,\norm{\cdot}).
\end{equation}
\end{lemma}

The following theorem provides an upper bound on the packing numbers of $\CF$ w.r.t.\ an $\mathcal{L}_1$-norm based on its VC dimension \cite[Corollary 3]{haussler1995sphere}. This result was also applied in the proof of Lemma \ref{lemma:covering-bound}.
\begin{theorem}\label{thm:vcbound-expect} 
For any set $\BX$, any probability distribution $\QX$ on $\BX$, any set $\CF$ of  $\QX$-measurable functions on $\BX$ taking values in the interval $[0,1]$ with $V_{\CF^{+}} < \infty$, and any $\varepsilon >0$
\begin{equation}
\mathcal{M}\big(\varepsilon, \CF,\norm{\cdot}_{\mathcal{L}_1(\QX)}\big) \leq e(V_{\CF^{+}}+1)  \bigg(  \frac{2e}{\varepsilon}\bigg)^{\!\!V_{\CF^{+}}}\hspace{-3mm}.
\end{equation}
\end{theorem}
\smallskip

In the proof of Lemma \ref{lemma:uniform-bound} we applied the result of \cite{gyorfi2002distribution}. 

\begin{theorem}\label{thm:expect} Assume that there exists $1\leq M < \infty$ such that $|Y| \leq M$ a.s. Let $f_n$ denote the truncated empirical risk minimizer to $[-M,M]$ over a set of functions $\CF_n$, then
\begin{equation}
\begin{aligned}
\EE \Big[ \int \,| \,f_n(x) -f^*(x)\,|^2\,\dd \QX(x)\Big] &\leq \frac{c_1}{n} + \frac{(c_2 +c_3 \log(n))V_{\CF_n^{+}}}{n}\\
&\quad + 2\inf_{f \in \CF_n} \int \,|f(x)- f^*(x)\,|^2 \dd \QX(x), 
\end{aligned}
\end{equation}
where $c_1=24\cdot214M^4(1+\log(42))$, $c_2 = 48\cdot214M^4\log(480eM^2)$ and $c_3 = 48\cdot214M^4$.
\end{theorem}

In the proof of Theorem \ref{theorem:algorithm-I} we used the following ULLN, which is an advanced version of \cite[Theorem 9.1]{gyorfi2002distribution}:
\begin{theorem}[Uniform Law of Large Numbers]\label{thm:ulln}
Let $X, X_1, \dots, X_n$ be i.i.d. random vectors taking values in $\RR^d$ and $\CH$ be a class of $\RR^d \to[0,B]$ type functions.
For any $n \in \NN$, and any $\varepsilon >0$
\begin{equation}
\begin{aligned}
&\PP\Big( \,\sup_{h \in \CH} \Big| \frac{1}{n}\sum_{i=1}^n h(X_i) - \EE[h(X)] \Big| >\varepsilon \,\Big) \leq 
&8 \EE\big[\CN_1(\varepsilon/8, \CH, \{X_1,\dots,X_n\})\big] e^{-\nicefrac{n\varepsilon^2}{(128B^2)}}.
\end{aligned}
\end{equation}
In addition, if for all $\varepsilon>0$ we have
\begin{equation}
\sum_{n=1}^\infty \EE\big[\, \CN_1(\varepsilon/8,\CH,\{X_1,\dots,X_n\})\,\big]e^{-\nicefrac{n\varepsilon^2}{(128B^2)}}<\infty,
\end{equation}
then the ULLN holds, that is
\begin{equation}
\sup_{h \in \CH} \Big| \,\frac{1}{n} \sum_{i=1}^n h(X_i) -\EE[h(X)] \,\Big| \xrightarrow{\,a.s.\,} 0.
\end{equation}
\end{theorem}

The weak $\mathcal{L}_2$ consistency of kNN estimates, \cite[Theorem 6.1]{gyorfi2002distribution}, is used in the proof of Theorem \ref{theorem:algorithm-II}.
\begin{theorem}\label{thm:weak-consistency}
Let $\{(X_i,Y_i)\}_{i=1}^n$ be an i.i.d. sample from the distribution of $(X,Y)$ such that $\PP\big(\,\norm{X- X'}_2 \neq \norm{X-X''}_2\,\big) =1$,
where $X'$ and $X''$ are independent copies of $X$.
Let $f^*$ denote the regression function.
If $k_n\to \infty$ and $\nicefrac{k_n}{n} \to 0$, then for estimator
\begin{equation}
f_n(x)\doteq \frac{1}{k_n} \sum_{i=1}^n Y_i \cdot \BI(\,X_i \in N(x,k_n)\,)
\end{equation}
the expected $\mathcal{L}_2$-loss tends to zero, that is
\begin{equation}
\EE \int \big| f_n(x) - f^*(x)\big|^2 \dd \QX(x) \to 0.
\end{equation}
\end{theorem}

McDiarmid's inequality is used in the proof of Theorem \ref{theorem:algorithm-II} and Theorem \ref{theorem:quantitative-knn}, \citep{mcdiarmid1989method}. 

\begin{theorem}[\textit{McDiarmid's inequality}]\label{mcdiarmid}
Let $X_1, \dots, X_n$ be independent random elements from set $A$ and $f: A^n \to \R$ be a function for which we have
\begin{equation}\label{eq:bounded-difference}
\begin{aligned}
\hspace*{-2.3mm}|f(x_1,\dots, x_n) - f(x_1, \dots, x_{i-1},y,x_{i+1},\dots, x_n)| \leq c_i
\end{aligned}
\end{equation}
for all $x_1,\dots,x_n,y \in A$ and $i \in [n]$. Then for all $\varepsilon \geq 0$
\begin{equation}
\begin{gathered}
\PP \,\big( \,|\,f(X_1,\dots,X_n) -\EE[f(X_1,\dots,X_n)]\,| \, \geq \varepsilon\,\big) \leq 2\exp\bigg( - \frac{2\varepsilon^2}{\sum_{i=1}^n c_i^2}\bigg).
\end{gathered}
\end{equation}
\end{theorem}

\end{appendix}

\bibliographystyle{imsart-number} 
\bibliography{reference}       

\end{document}